\documentclass[sigplan,10pt]{acmart}  %

\acmYear{2024}\copyrightyear{2024}
\setcopyright{acmlicensed}
\acmConference[SOSP '24]{ACM SIGOPS 30th Symposium on Operating Systems Principles}{November 4--6, 2024}{Austin, TX, USA}
\acmBooktitle{ACM SIGOPS 30th Symposium on Operating Systems Principles (SOSP '24), November 4--6, 2024, Austin, TX, USA}
\acmDOI{10.1145/3694715.3695970}
\acmISBN{979-8-4007-1251-7/24/11}

\usepackage{url, cite}
\usepackage{xspace}

\usepackage{microtype}
\usepackage{graphicx}
\usepackage{subcaption}
\usepackage{booktabs} %
\usepackage{multirow}
\usepackage[inline]{enumitem}

\usepackage{hyperref}
\usepackage{xspace}
\usepackage{tabularx}
\usepackage{tikz}
\usetikzlibrary{positioning}
\newcommand*\circled[1]{\tikz[baseline=(char.base)]{
		\node[shape=circle,draw,inner sep=0pt] (char) {#1};}}
\newcommand*\blackcircled[1]{\tikz[baseline=(char.base)]{
		\node[shape=circle,draw,fill=black,inner sep=0pt] (char) {\textcolor{white}{#1}};}}

\usepackage{subcaption}
\usepackage{url, cite}
 \usepackage{verbatim}
\usepackage{multirow}
\usepackage[makeroom]{cancel}
\usepackage{dsfont}
\usepackage{amsmath}
\usepackage{bbold}

\usepackage{tikz}  

\usepackage{tabularx}
\usepackage{color, colortbl}
\usepackage{hhline}

\definecolor{commentgreen}{rgb}{0, 0.5, 0}
\usepackage[
  vlined,
  linesnumbered,
]{algorithm2e}
\DontPrintSemicolon{}
\SetKwComment{Comment}{$\triangleright$\ }{}

\SetCommentSty{mycommfont}
\newcommand{\algrule}[1][.7pt]{\par\vskip.5\baselineskip\hrule height #1\par\vskip.5\baselineskip}
\let\oldnl\nl
\newcommand{\nonl}{\renewcommand{\nl}{\let\nl\oldnl}}
 
\usepackage[frozencache]{minted}
\usepackage{xcolor}

\newcommand{\parabf}[1]{\medskip\noindent\textbf{#1}}

\newcommand{\revised}[1]{#1}

\newcommand{\mosharaf}[1]{\textcolor{blue}{\{MC: #1\}}}

\newcommand{\jw}[1]{\textcolor[rgb]{0.25,0.65,0.1}{\{JW: #1\}}}
\newcommand{\yile}[1]{\textcolor{orange}{\{YG: #1\}}}
\newcommand{\insu}[1]{\textcolor[rgb]{0.56, 0.27, 0.52}{\{IS: #1\}}}

\renewcommand{\mosharaf}[1]{}
\renewcommand{\jw}[1]{}
\renewcommand{\yile}[1]{}
\renewcommand{\insu}[1]{}

\def\ie{{i.e.\xspace}}
\def\eg{{e.g.\xspace}}

\usepackage{hyperref}
\hypersetup{
  colorlinks=true,      %
  linkcolor=blue,       %
  citecolor=magenta,    %
  filecolor=cyan,       %
  urlcolor=red          %
}

\graphicspath{{figures/}}
\newenvironment{denseitemize}{
	\begin{itemize}[topsep=2pt, partopsep=0pt, leftmargin=1.5em]
		\setlength{\itemsep}{2pt}
		\setlength{\parskip}{0pt}
		\setlength{\parsep}{0pt}
	}{\end{itemize}}

\newenvironment{denseenum}{
	\begin{enumerate}[topsep=2pt, partopsep=0pt, leftmargin=1.5em]
		\setlength{\itemsep}{2pt}
		\setlength{\parskip}{0pt}
		\setlength{\parsep}{0pt}
	}{\end{enumerate}}

\renewcommand{\cite}{\citep}
\renewcommand{\paragraph}{\parabf}

\begin{document}

\title{Reducing Energy Bloat in Large Model Training}

\author{Jae-Won Chung$^{\text{1}}$\enskip Yile Gu$^{\text{1}, \text{2}}$\enskip Insu Jang$^{\text{1}}$\enskip Luoxi Meng$^{\text{1}, \text{3}}$\enskip Nikhil Bansal$^{\text{1}}$\enskip Mosharaf Chowdhury$^{\text{1}}$}
\affiliation{\vspace{1mm} $^{\text{1}}$University of Michigan \enskip $^{\text{2}}$University of Washington \enskip $^{\text{3}}$University of California, San Diego \country{}}

\begin{abstract}
Training large AI models on numerous GPUs consumes a massive amount of energy, making power delivery one of the largest limiting factors in building and operating datacenters for AI workloads.
However, we observe that not all energy consumed during training directly contributes to end-to-end \revised{throughput; a} significant portion can be removed without slowing down training.
We call this portion \emph{energy bloat}.

In this work, we identify two independent sources of energy bloat in large model training and propose Perseus, a training system that mitigates both.
To do this, Perseus obtains the \revised{time--energy tradeoff frontier of a} large model training job using an efficient graph cut-based \revised{algorithm, and schedules computation energy consumption} across time to reduce both types of energy bloat.
Evaluation on large \revised{models, including GPT-3 and Bloom,} shows that Perseus reduces the energy consumption of large model training by up to 30\% without any throughput loss or hardware modification.\footnote{Perseus is open-source as part of Zeus~\cite{zeus-nsdi23} at \url{https://ml.energy/zeus}.}  %

\end{abstract}



\maketitle

\pagestyle{plain}
\pagenumbering{gobble}

\section{Introduction}\label{sec:intro}

As deep neural networks (DNNs) continue to grow in model and dataset size~\cite{scalinglaws-arxiv20,chinchilla-neurips22}, the energy consumption of large model training is increasing as well.
For instance, training GPT-3~\cite{gpt3} reportedly consumed 1.3 GWh~\cite{patterson2021carbon}.
\revised{Then, this was dwarfed by Amazon's training of a 200B model, which consumed about 11.9 GWh~\cite{hamilton2024constraint}---enough to power more than 1,000 average US households for a year~\cite{us-household}.}
Such energy-intensive large model training not only inflates datacenter operational expenses, but also made power delivery a primary challenge in building datacenters today~\cite{cbre2023,mckinsey2023,cbre2024,aienergy-joule24,openai-keynote-hotchips24}.

\looseness=-1
Despite recent works on accelerating large model training~\cite{megatronlm-sc21,alpa-osdi22,merak-tpds23}, energy optimization \revised{remains} an open challenge~\cite{greenai-cacm20,patterson2021carbon}.
While energy optimization is well-studied in the hardware community~\cite{eyeriss-jssc16,vivienne17efficient,puma-asplos19,kws-jssc24}, the power bottleneck of recent datacenters~\cite{cbre2023,mckinsey2023,cbre2024,dvfs-boosting-asplos24,openai-keynote-hotchips24} shows that efficiency gains from hardware advancement alone are not sufficient to sustain the growing demand for AI compute.
In light of this, recent works show that software can play a significant role in energy optimization by capturing application characteristics that general-purpose hardware cannot (\eg, no need to finish computation before the deadline), bringing hardware-agnostic energy-efficiency gains~\cite{gpoeo-tpds21,zeus-nsdi23,chase-ccai23,envpipe-atc23,crosslayer-energy-eecs24}.

\looseness=-1
In this paper, we seek a \emph{software} method that reduces the energy consumption of large model training \emph{without} slowdown, thereby also reducing average power draw.
To that end, we identify \revised{\emph{energy bloat}, the portion of energy consumption that can be removed without slowdown in} software systems for large model training.
We find two independent sources of energy \revised{bloat---\emph{intrinsic} and \emph{extrinsic}---and} propose a single optimization framework that minimizes both.

Intrinsic energy bloat comes from computation imbalance when a large model is distributed across multiple GPUs with pipeline parallelism (\S\ref{sec:motivation-intrinsic}).
Balancing the amount of computation in each pipeline stage is an important problem for distributed execution planning~\cite{pipedream-sosp19,dapple-ppopp21,autopipe-icml21,alpa-osdi22}, but perfectly balancing every stage is not always possible because layers in a DNN are coarse-grained tensor operations with varying amounts of computation.
\revised{When stages have unequal computation times}, those not on the \emph{critical path} of computation run needlessly fast---that is, they consume energy that does not contribute to the overall training throughput.
Such intrinsic energy bloat opens up the opportunity to \emph{precisely} slow down each non-critical computation in the pipeline such that the length of the critical path does not change.

Extrinsic energy bloat, in contrast, arises when multiple pipelines run in parallel in a synchronous fashion, and one or more pipelines run slower than the rest (\S\ref{sec:motivation-extrinsic}).
Root causes behind such slowdowns are varied, including power/thermal throttling~\cite{A100powerthermal-energies21,gpupower-cal23,thunderbolt-osdi20,mvpp-asplos20,polca-asplos24}, I/O bottlenecks in the storage/network~\cite{datastalls-vldb21,recsys-dsi-isca22,antdt-icde24}, \revised{and} hardware/software failures~\cite{bamboo-nsdi23,oobleck-sosp23,recycle-sosp24}, and the likelihood of their presence increases with the scale and duration of training~\cite{largescalefailures-sc17,philly-atc19,mlaas-nsdi22}.
All pipelines running faster than the \emph{straggler pipeline} are needlessly \revised{fast, wasting} energy that does not affect the overall training throughput.
Thus, we can slow down entire pipelines without delaying gradient synchronization.

In this work, we propose \emph{Perseus}, which formulates a unified optimization framework to remove both intrinsic and extrinsic energy bloat from large model training (\S\ref{sec:overview}).
At its core, Perseus efficiently pre-characterizes the entire \revised{time--energy tradeoff frontier of a training iteration}, allowing it to minimize intrinsic bloat under normal operation and to mitigate extrinsic bloat arising from stragglers.
Existing works fall short on both fronts.
EnvPipe~\cite{envpipe-atc23} is limited to intrinsic bloat reduction with a point solution that leads to suboptimal energy reduction.
Zeus~\cite{zeus-nsdi23}, in contrast, ignores intrinsic bloat as it only considers single-GPU training, which also renders its \revised{time--energy} frontier suboptimal for large models.

\looseness=-1
We show that characterizing the optimal time--energy Pareto frontier is NP-hard not only to solve, but also to approximate within any constant factor. 
Given this impasse, we propose an efficient algorithm that optimally solves a relaxed problem instead (\S\ref{sec:design}).
To do so, Perseus represents one training iteration as a directed acyclic graph (DAG) \revised{of forward and backward computations in each pipeline stage}.
\revised{Then, Perseus efficiently generates all \emph{energy schedules}, defined as the planned time and energy consumption of each computation, that are on the \revised{time--energy frontier} using a graph cut-based algorithm that iteratively \emph{crawls up} the frontier from the bottom.}
Minimizing intrinsic and/or extrinsic energy bloat is then as simple as choosing the appropriate energy schedule from the \revised{pre-characterized time--energy frontier}.

\looseness=-1
Perseus \revised{consists of a client library and a server (\S\ref{sec:implementation}).
The client library integrates with a large model training framework and accelerator to measure computation energy consumption and control accelerator speed.
The server produces optimized energy schedules, using the abstract computation DAG and time/energy measurements provided by the client.}

Evaluation on large models \revised{(GPT-3~\cite{gpt3}, BERT~\cite{bert}, T5~\cite{t5-jmlr20}, Bloom~\cite{bloom-arxiv22}, and a scaled-up version of Wide-ResNet~\cite{wideresnet-bmvc16})}, shows that Perseus is able to reduce per-iteration energy consumption by up to 30\% with negligible or no slowdown, reducing energy consumption and average power draw (\S\ref{sec:eval}).

Overall, we make the following contributions in this paper:
\begin{denseitemize}
  \item We identify intrinsic and extrinsic energy bloat in large model training, fundamentally caused by computation time imbalance at different levels.
  
  \item We propose Perseus, a software-only energy optimization system that reduces energy bloat through a unified optimization framework and a graph cut-based algorithm.
  
  \item We evaluate Perseus on a diverse set of large model workloads and show that it significantly reduces energy bloat, bringing hardware-agnostic energy savings.
\end{denseitemize}

\section{Motivation}\label{sec:motivation}

\begin{figure}[t]
  \centering
  \subfloat[Execution timeline of one training iteration]{
    \includegraphics[width=0.465\textwidth]{
      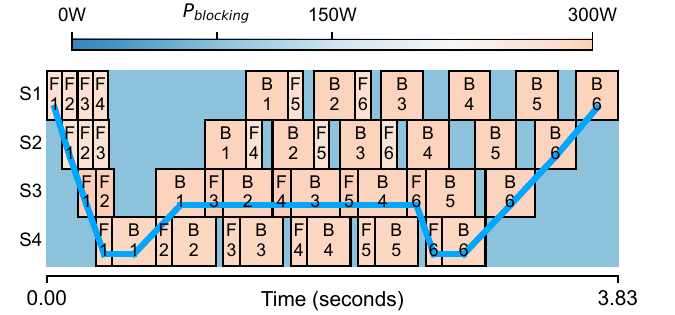
    }\label{fig:motivation-lowtime-maxfreq}
  }

  \subfloat[Execution timeline with reduced intrinsic energy bloat]{
    \includegraphics[trim={0 0 0 25},clip,width=0.465\textwidth]{
      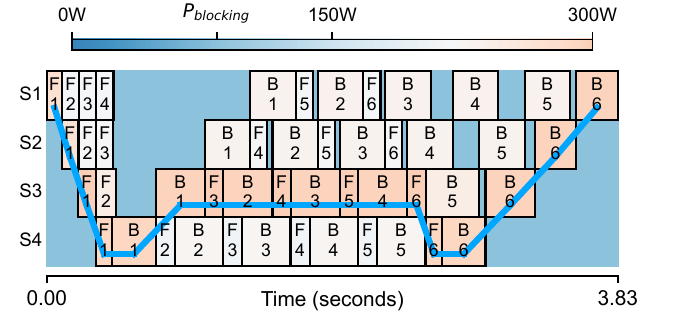
    }\label{fig:motivation-lowtime-lastpd}
  }
  \caption{One training iteration of GPT-3 1.3B with 4 pipeline stages and 6 microbatches on NVIDIA A100 \revised{GPUs, drawn} to scale.
    For example, F5 and B5 in the S2 row denote forward and backward for the fifth microbatch on Stage 2. 
    The critical path is traced with a blue line.
    Colors show power consumption.
    Other models are visualized in Appendix~\ref{sec:appendix-intrinsic-bloat}.}
\end{figure}

First, we provide necessary background regarding large model training (\S\ref{sec:motivation-large-model-training}).
Then, we introduce intrinsic (\S\ref{sec:motivation-intrinsic}) and extrinsic (\S\ref{sec:motivation-extrinsic}) energy bloat present in large model training, and discuss opportunities for energy reduction (\S\ref{sec:motivation-opportunity}).

\subsection{Large Model Training}\label{sec:motivation-large-model-training}

\looseness=-1
Large model training is mostly dominated by 3D (data, tensor, and pipeline) parallelism~\cite{megatronlm-sc21,megatron-turing-arxiv22,bloom-arxiv22,merak-tpds23,falcon-arxiv23}.
\revised{Especially, pipeline parallelism partitions a large model into multiple \emph{stages} and its training batch into \emph{microbatches}, and pipelines forward and backward computations through the stages.}
Then, such pipelines are replicated to perform data parallel training.
Pipelines can only move on to the next iteration after every pipeline has finished and synchronized gradients.

\subsection{Intrinsic Energy Bloat}\label{sec:motivation-intrinsic}

\looseness=-1
We profile GPT-3 1.3B on NVIDIA A100 GPUs and visualize the timeline of one training iteration in Figure~\ref{fig:motivation-lowtime-maxfreq}.
In addition to the familiar bubbles in the 1F1B schedule~\cite{megatronlm-sc21}, we observe \emph{gaps} \revised{between} forward and backward computations, where the GPU is simply blocking on communication with an adjacent stage.
Such gaps exist because the computation time of each pipeline stage is not perfectly balanced.
Partitioning stages in a balanced manner is an important problem in distributed execution planning~\cite{gpipe-neurips19,pipedream-sosp19,dapple-ppopp21,alpa-osdi22}, but \emph{perfect} balancing is difficult because DNNs are essentially a sequence of coarse-grained tensor operations with varying \revised{sizes}.

\begin{table}[t!]
  \footnotesize
	\centering
  \begin{tabular}{lrrr}
    \toprule
    \multirow{2}{*}{\textbf{Model}} & \multirow{2}{*}{\textbf{\# Parameters}} & \multicolumn{2}{c}{\textbf{Imbalance Ratio}} \\ 
                   & & 4 stages & 8 stages \\
    \midrule
    \multirow{4}{*}{GPT-3~\cite{gpt3}}          &   3B  & 1.13  & 1.25  \\
                                                &   7B  & 1.11  & 1.23  \\
                                                &  13B  & 1.08  & 1.17  \\
                                                & 175B  & 1.02  & 1.03  \\
    \midrule
    \multirow{3}{*}{Bloom~\cite{bloom-arxiv22}} &   3B  & 1.13  & 1.25  \\
                                                &   7B  & 1.13  & 1.25  \\
                                                & 176B  & 1.05  & 1.10  \\
    \midrule
    \multirow{2}{*}{BERT~\cite{bert}}           & 0.1B  & 1.33  & 2.00  \\
                                                & 0.3B  & 1.17  & 1.33  \\
    \midrule
    \multirow{3}{*}{T5~\cite{t5-jmlr20}}        & 0.2B  & 1.19  & 1.50  \\
                                                & 0.7B  & 1.05  & 1.11  \\
                                                & 2.9B  & 1.06  & 1.16  \\
    \midrule
    Wide-ResNet50~\cite{wideresnet-bmvc16}      & 0.8B  & 1.23  & 1.46  \\
    Wide-ResNet101~\cite{wideresnet-bmvc16}     & 1.5B  & 1.09  & 1.25  \\
    \bottomrule
  \end{tabular}
  \caption{Forward latency ratio of the longest to the shortest stage on A100 GPUs. 1.00 would mean perfect balance.}\label{tab:motivation-imbalance}
\end{table}

To understand the amount of possible pipeline stage imbalance, we exhaustively searched for the pipeline partition with the smallest imbalance ratio, defined as the ratio of the longest stage forward computation latency to the shortest.\footnote{For Transformer-based models, we partition at the granularity of Transformer layers. For Wide-ResNet, we partition at the granularity of bottleneck layers, which are three convolution layers wrapped with a skip connection.}
Table~\ref{tab:motivation-imbalance} lists the minimum imbalance ratio for various models, which shows that perfect balance is difficult to achieve.
See Appendix~\ref{sec:appendix-workloads} for partitioning details and sources of imbalance.

Given stage imbalance, \revised{not all forward and backward computations are on the \emph{critical path} of computation (Figure~\ref{fig:motivation-lowtime-maxfreq}).}
This means that non-critical computations running at their maximum speed \revised{are} not contributing to faster iteration time, and thus simply wasting energy.
We call this intrinsic energy bloat, which can be reduced by precisely slowing down each non-critical computation without lengthening the critical path (Figure~\ref{fig:motivation-lowtime-lastpd}).
Although seemingly simple, this problem is not only NP-hard to solve, but also NP-hard to even \emph{approximate} to any constant factor~\cite{svensson2012hardness}. %

\subsection{Extrinsic Energy Bloat}\label{sec:motivation-extrinsic}

Numerous replicas of the same pipeline run in a data parallel fashion in large model training.
Because every pipeline must synchronize gradients at the end, if \revised{one} pipeline runs slower, \emph{all other} pipelines must wait until the straggler pipeline finishes (Figure~\ref{fig:motivation-extrinsic-before}).
Since the straggler pipeline determines end-to-end iteration time, all other pipelines running at their fastest possible iteration time \revised{are} wasteful.
We call this extrinsic energy bloat, because unlike intrinsic energy bloat, its cause is \emph{extrinsic} to the \revised{training} pipeline.
To reduce extrinsic bloat while keeping intrinsic bloat low, one can determine the energy-optimal iteration time for non-straggler pipelines and precisely slow down computations so that non-straggler pipelines attain that iteration time (Figure~\ref{fig:motivation-extrinsic-after}).

\begin{figure}[!t]
  \centering
  \subfloat[Extrinsic energy bloat caused by a straggler]{
    \includegraphics[width=0.38\textwidth]{
      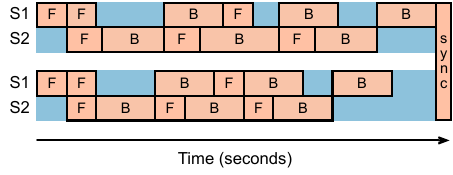
    }\label{fig:motivation-extrinsic-before}
  }
  \vspace{2mm}

  \subfloat[Reduced intrinsic and extrinsic energy bloat]{
    \includegraphics[width=0.38\textwidth]{
      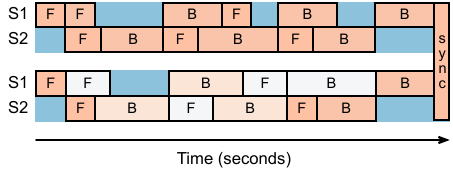
    }\label{fig:motivation-extrinsic-after}
  }
  \caption{\revised{Among two data parallel pipelines, the first one \revised{becomes} a straggler.}
    The non-straggler pipeline \revised{causes extrinsic energy bloat by running as fast as possible (a), which can be reduced by precisely slowing it down (b).}
  }\label{fig:motivation-extrinsic}
    \vspace{3mm}
\end{figure}

Stragglers arise from numerous sources.
Thermal or power throttling in a datacenter can result in 10--50\% slowdown~\cite{A100powerthermal-energies21,gpupower-cal23,thunderbolt-osdi20,mvpp-asplos20,polca-asplos24}, and I/O bottlenecks in the storage or network can be longer than GPU computation by up to $4\times$~\cite{datastalls-vldb21,recsys-dsi-isca22,antdt-icde24}, acting like a persistent straggler pipeline.
Recent failure-resilient training frameworks~\cite{bamboo-nsdi23,oobleck-sosp23,recycle-sosp24} deploy \emph{heterogeneous} pipelines, introducing non-uniform iteration times.
With increasing job and infrastructure scale, the probability of encountering stragglers increases~\cite{largescalefailures-sc17,philly-atc19,mlaas-nsdi22,megascale-nsdi24,tpuv4-nsdi24}.

In this work, we focus on stragglers that are known to and anticipated by the training infrastructure, generally because they were created by the infrastructure itself (\eg, power and thermal throttling, non-compute bottlenecks, fault-tolerant planning).
Such stragglers also tend to \revised{persist} beyond typical training iteration times.
\revised{Therefore, Perseus focuses on planning time and energy consumption across time and allowing quick adaptation, assuming that information about stragglers is available.}

\subsection{Potential Benefits of Reducing Energy Bloat}\label{sec:motivation-opportunity}

To gauge potential energy savings, we measure the energy savings \revised{achieved by} slowing down \emph{every} computation in the pipeline to their minimum-energy frequencies.
This will slow down iteration time, but can act as an upper bound for energy savings.
For our workloads in Section~\ref{sec:eval-reducing-energy-bloat}, this gives on average 16\% and 27\% energy reduction on A100 and A40 GPUs, respectively.
\revised{Section~\ref{sec:eval-reducing-energy-bloat} shows} that Perseus can realize most of the potential savings with negligible \revised{slowdown.}

\begin{figure*}[!ht]
    \centering
    
    \subfloat[$T' = T_{\min}$]{
      \includegraphics[width=0.32\textwidth]{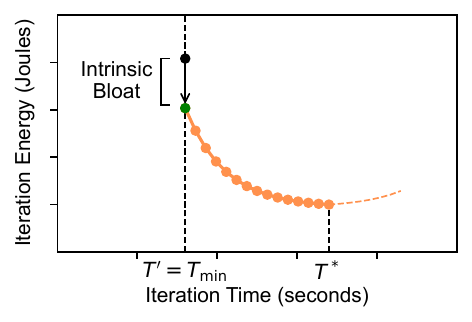}\label{fig:overview-frontier1}
    }
    \subfloat[$ T_{\min} < T' \leq T^*$]{
      \includegraphics[width=0.32\textwidth]{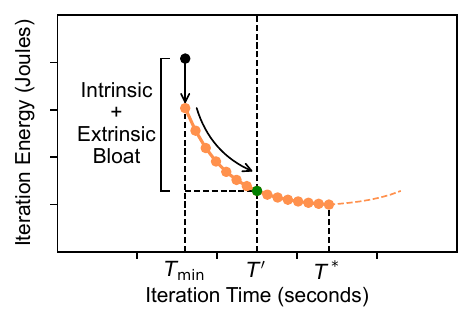}\label{fig:overview-frontier2}
    }
    \subfloat[$T^* < T'$]{
      \includegraphics[width=0.32\textwidth]{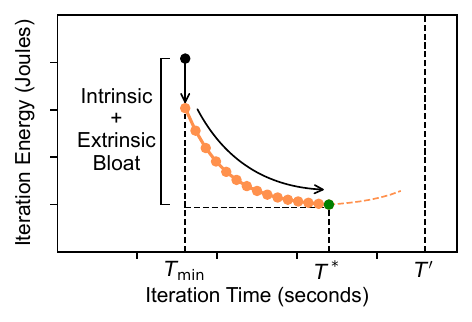}\label{fig:overview-frontier3}
    }
    
    \caption{Three cases \revised{showing} where the straggler pipeline's iteration time $T'$ can be.
    $T_{\min}$ and $T^*$ are the shortest and longest iteration \revised{times} on the \revised{time--energy frontier}.
    The black dot is when all computations run at the maximum speed, which wastes energy.
    The green dot is the energy-optimal iteration time of the non-straggler pipeline.
    Solid orange dots make up the \revised{frontier}, and the orange dotted line shows that iteration energy increases beyond $T^*$.
    }\label{fig:overview-frontiers}
\end{figure*}

\section{Perseus Overview}\label{sec:overview}

We first present Perseus's unified optimization framework that aims to remove both types of energy bloat (\S\ref{sec:overview-optimization}), and then walk through the workflow of Perseus (\S\ref{sec:overview-workflow}).

\subsection{Unified Optimization Framework}\label{sec:overview-optimization}

Intuitively, slowing down computations selectively in a training pipeline without affecting its critical path will keep the same iteration time while reducing its energy consumption (\S\ref{sec:motivation-intrinsic}).
Furthermore, when stragglers emerge, slowing down computations in a non-straggler pipeline without making it a straggler itself will reduce energy consumption even more (\S\ref{sec:motivation-extrinsic}).
We formalize these two intuitions into a unified optimization framework and derive a universal prescription for a non-straggler pipeline's \emph{energy-optimal} iteration time.

Our goal is to minimize a pipeline's energy consumption by controlling the execution speed of each computation in the pipeline.
In doing so, we can slow down a pipeline's iteration time \revised{up to} the straggler's iteration time $T'$:
\begin{equation}
  \begin{aligned}
    \min_F & \quad \textrm{Energy}(F)  \\ 
   \textrm{s.t.} & \quad \textrm{Time}(F) \leq T'
  \end{aligned}
	\label{eq:overview-optimization-problem}
\end{equation}
where $F$ is \revised{an assignment of GPU frequencies}\footnote{The SM frequency of NVIDIA GPUs can be set via NVML~\cite{nvml} in around 10 ms, which is much shorter than typical large model computation latencies. Locking the GPU's frequency provides deterministic computation latency~\cite{nvml,clockwork-osdi20}, making it suitable for tightly planning and packing execution over time.} \revised{to each forward and backward computation in the pipeline}, and $\textrm{Time}(F)$ and $\textrm{Energy}(F)$ are the iteration time and energy consumption of the pipeline when executed with $F$, respectively.
Changing $F$ will lead to different values of $\textrm{Time}(F)$ and $\textrm{Energy}(F)$, but we are only interested in ($\textrm{Time}(F)$, $\textrm{Energy}(F)$) points that are on the \revised{time--energy tradeoff frontier}.

Now, let us assume we have a fully characterized \revised{time--energy frontier}, bookended by $T_{\min}$ and $T^*$ (\revised{Section~\ref{sec:design} is dedicated to describing how}).
$T_{\min}$ is the shortest iteration time on the frontier, which is the same as the iteration time of running every computation at the maximum speed, and $T^*$ is the iteration time with minimum energy consumption, which is when each computation runs at the frequency that consumes the least amount of energy for that computation.\footnote{This is typically not the lowest frequency, because computations running with very low frequencies incur more latency increase than power reduction, resulting in \emph{higher} energy consumption.}
\revised{Figure~\ref{fig:overview-frontiers} shows the three possible cases regarding where the straggler's iteration time $T'$ can be:}
\begin{denseenum}
  \item Figure~\ref{fig:overview-frontier1}: When there are no stragglers, we simply select the point on the frontier with iteration time $T_{\min}$, which reduces \emph{only intrinsic energy bloat}.

  \item Figure~\ref{fig:overview-frontier2}: When a moderately slow straggler is detected, we \emph{additionally reduce extrinsic energy bloat while keeping intrinsic bloat low} by slowing down all non-straggler pipelines until $T'$, using up all the slack time.

  \item Figure~\ref{fig:overview-frontier3}: Finally, the straggler's iteration time may go beyond the minimum-energy point $T^*$ on the frontier. 
    In this case, we only slow down non-stragglers until $T^*$, because going past $T^*$ will instead \emph{increase} energy.  
\end{denseenum}

The three cases can be merged into one universal prescription for the pipeline's energy-optimal iteration time:
\begin{equation}
\begin{aligned}
T_{\textrm{opt}} = \min (T^*, T').
\end{aligned}
\label{eq:overview-optimal-time}
\end{equation}
\revised{Therefore, when a straggler emerges (\ie, $T_{\min} < T'$), Perseus can compute $T_\textrm{opt}$ using Equation~\ref{eq:overview-optimal-time} and quickly look up the frequency plan $F_\textrm{opt}$ that leads to iteration time $T_\textrm{opt}$ using the pre-characterized time--energy frontier.}

\revised{Finally, we note that, unlike other problem settings that do not consider energy consumption, fully utilizing all the slack time created by the straggler \emph{is not always energy-optimal}; being too fast or too slow can both waste energy.}

\subsection{Perseus Architecture}\label{sec:overview-workflow}

\paragraph{Energy Schedule.}
Perseus represents each iteration of the training pipeline as a static directed acyclic graph (DAG), where nodes are forward and backward computations in each stage and edges are dependencies between computations.
Each node on the computation DAG is annotated with its planned time and energy consumption, which we call the \emph{energy schedule}.
Perseus realizes an energy schedule by executing each computation with a specific GPU frequency.

\paragraph{System Components.}
\revised{Perseus's architecture is shown in Figure~\ref{fig:overview-workflow}.
Perseus consists of} a framework- and accelerator-agnostic server and a framework-integrated and accelerator-specific client. 
The server is a cluster-wide singleton.
For various training jobs, the server pre-characterizes the \revised{time--energy} frontier of one iteration (\S\ref{sec:design}) and caches energy schedules for fast lookup.
The client profiles pipeline computations online during training and realizes \revised{energy schedules} by setting the GPU's frequency during runtime (\S\ref{sec:implementation}).

\begin{figure}[t]
  \centering
  \includegraphics[width=0.40\textwidth]{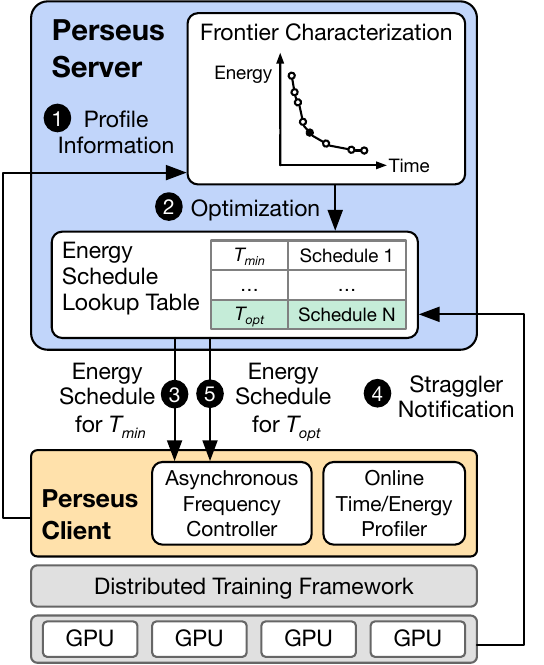}
  \caption{Perseus architecture and workflow.}\label{fig:overview-workflow}
  \vspace{4mm}
\end{figure}

\paragraph{Training Lifecycle.}
For the Perseus server, a training job is primarily specified by its computation DAG for one training iteration.
When the job begins execution, \blackcircled{1} the Perseus client invokes its \revised{online time--energy profiler} (\S\ref{sec:implementation}) to measure the time and energy of each forward and backward computation on each supported frequency.
Profiling is done \emph{in vivo} during the initial tens of training iterations.

Upon receiving profiling results, \blackcircled{2} the server begins \emph{asynchronously} characterizing the \revised{time--energy} frontier (\S\ref{sec:design}) while training continues.
When characterization finishes, energy schedules on the frontier are saved in a lookup table indexed by $T'$.
Then, \blackcircled{3} the energy schedule corresponding to $T_{\min}$ is deployed to the client.
Energy schedules are realized by the client's asynchronous frequency controller, integrated into the training framework (\S\ref{sec:implementation}).

During training, \blackcircled{4} the training infrastructure (\eg, datacenter rack power/temperature manager) notifies the Perseus server of a straggler and its degree of slowdown.
The server then \blackcircled{5} quickly reacts to this by looking up the energy schedule corresponding to the anticipated straggler iteration time (the one with iteration time $T_{\textrm{opt}}$), and deploys it to the client.

\section{Characterizing the Time--Energy Frontier}\label{sec:design}

In this section, we describe our algorithm to efficiently obtain the \revised{time--energy tradeoff frontier} for a training pipeline in detail.
We first formulate the problem, show that it is NP-hard, and describe a relaxed version (\S\ref{sec:design-problem-formulation}).
Then, we provide an overview of our algorithm (\S\ref{sec:design-algorithm-overview}) and describe the core subroutine in our algorithm (\S\ref{sec:design-reducing-time-optimally}).
Finally, we extend our algorithm to support 3D/hybrid parallelism, constant-time operations, and diverse pipeline schedules (\S\ref{sec:design-extensions}).

\subsection{Problem Formulation}\label{sec:design-problem-formulation}

\paragraph{Expression for Energy Consumption.}
The energy consumption of a pipeline is not only from computation; it is the sum of three parts:
(1) Computation;
(2) Blocking on communication between computations; and
(3) Blocking on communication until the straggler pipeline finishes:

\begin{equation}
  \label{eq:design-energy-formulation}
  \medmuskip=3mu
  \thinmuskip=-2mu
  \begin{aligned}
    & \sum_{i}{e_{i}(f_i)} + P_{\textrm{blocking}} ( N \cdot T - \sum_i{t_i(f_i)} ) +  P_{\textrm{blocking}} \cdot N \cdot (T' - T)\\
    & = \underset{\circled{1}}{\underline{
          \sum_{i}{\left( e_{i}(f_i) - P_{\textrm{blocking}} \cdot t_i(f_i) \right)}
        }} + 
        \underset{\circled{2}}{\underline{
          \vphantom{\sum_i}P_{\textrm{blocking}} \cdot N \cdot T'
        }}
  \end{aligned}
\end{equation}
where $P_{\textrm{blocking}}$ is the power consumption of the GPU when it is blocking on communication, $N$ is the number of pipeline stages, and $t_i(f_i)$ and $e_{i}(f_i)$ are the time and energy consumption of computation $i$ with frequency $f_i$, respectively.\footnote{An assumption here is that $P_{\textrm{blocking}}$ is constant, as a GPU blocking on communication is busy-looping inside a NCCL kernel without heavy computation utilization.}

As derived in Section~\ref{sec:overview-optimization}, given straggler iteration time $T'$, we draw a vertical line on the \revised{time--energy tradeoff frontier} ($\textrm{Time}(F)$ vs. \circled{1}+\circled{2}) at $T_\textrm{opt}$ and find $F_\textrm{opt}$ where the two lines intersect.
\revised{Equation~\ref{eq:design-energy-formulation} shows that the time--energy frontier of a pipeline depends on the straggler's iteration time $T'$ \emph{only} in the second term \circled{2}, which is merely an \emph{upward shift} of the frontier.} %
Therefore, if we characterize the \revised{tradeoff frontier} of $\textrm{Time}(F)$ vs. \circled{1}, that frontier can be used to find $T_{\textrm{opt}}$ and $F_{\textrm{opt}}$ for any \revised{straggler iteration time} $T'$. %
Thus, we define
\begin{equation}
  \textrm{Energy}(F) = \sum_{i}{\left( e_{i}(f_i) - P_{\textrm{blocking}} \cdot t_i(f_i) \right)}
\end{equation}
and characterize the frontier of $\textrm{Time}(F)$ vs. $\textrm{Energy}(F)$.

\paragraph{Finding the Time--Energy Frontier.}
Finding \emph{one point} on the Pareto-optimal tradeoff frontier with iteration time $T$ is equivalent to solving the following optimization problem:
\begin{equation}
  \begin{aligned}
    \min_F & \quad \textrm{Energy}(F) \\
    \textrm{s.t.} & \quad \textrm{Time}(F) \leq T
  \end{aligned}
\end{equation}
We call this problem \emph{Pipeline Energy Minimization} (PEM).

\begin{theorem}
  Pipeline Energy Minimization is NP-hard.
\end{theorem}
\begin{proof}
  Reduction from Knapsack. Details in Appendix~\ref{sec:appendix-np-hardness}.
\end{proof}

The complete Pareto-optimal tradeoff frontier can be obtained by solving PEM for all $T \in [T_{\min}, T^*]$, which is clearly intractable.
Therefore, we seek an appropriate relaxation of the problem that will yield a \emph{nearly} Pareto-optimal frontier.

One of the reasons PEM is NP-hard is because it is a \emph{discrete} optimization problem where the possible choices of computation time and energy are discrete, which is in turn because GPUs only support discrete frequencies (\eg, in 15 MHz steps and nothing in the middle).
However, if frequency choices were \emph{continuous}, the problem is exactly and efficiently solvable~\cite{skutella1998}.
This is akin to integer linear programs becoming tractable when relaxed to linear programs.

The transform from the original problem to the relaxed version is done by fitting a \emph{continuous} exponential function ($a \cdot e^{bt} + c$) to Pareto-optimal computation time and energy measurements for each forward and backward computation.
We choose the exponential function due to its inherent flexibility and natural fit to data (\revised{more} details in Appendix~\ref{sec:appendix-relaxation}).
We show in Section~\ref{sec:eval-reducing-energy-bloat} that this relaxation produces high-quality approximate solutions that \revised{realize} most of the opportunity for savings.
Solving the relaxed problem returns the time and energy consumption planned for each computation in the pipeline, or the energy schedule.
Then, this is transformed back to a feasible solution of the original problem, which is the set of GPU frequencies $F$.

\subsection{Iteratively Discovering the Frontier}\label{sec:design-algorithm-overview}

\begin{figure}
\centering
  \includegraphics[width=0.37\textwidth]{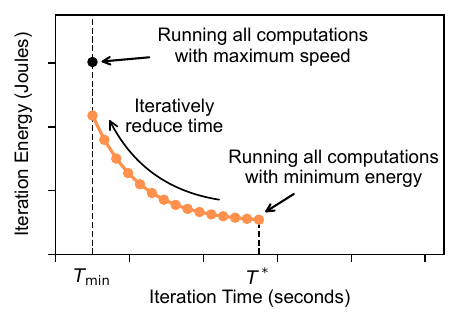}
  \vspace{-3mm}
  \caption{Starting from the energy schedule that consumes the minimum energy, we iteratively reduce its iteration time to trace up and iteratively discover the tradeoff frontier.}\label{fig:design-frontier-navigation}
\end{figure}

Now, we first describe our iterative strategy of characterizing the frontier, and dive deeper into one iteration in Section~\ref{sec:design-reducing-time-optimally}.

Although our relaxed problem is no longer NP-hard, solving it for each $T' \in [T_{\min}, T^*]$ from scratch is inefficient.
Instead, what if we can \emph{tweak} an existing \revised{schedule already on the frontier} to generate its \emph{neighbor} energy schedule on the frontier?
Then, we can start from one end of the frontier and trace along to the other end, discovering \revised{fine-grained optimized energy schedules}. %

Figure~\ref{fig:design-frontier-navigation} visualizes our strategy.
We start from the rightmost point $T^*$ that consumes the \emph{minimum energy}, which is simply running every computation with the minimum energy.\footnote{The minimum energy consumption for each computation type can be queried from the computation time/energy profiling information (\S\ref{sec:implementation}).}
This energy schedule is in fact Pareto-optimal because there are no other schedules that achieve the same energy with faster time.
Then, we iteratively reduce iteration time by unit time $\tau$ (\eg, 1 ms) while increasing total energy \emph{minimally}, which gives us the neighbor energy schedule on the frontier.\footnote{$\tau$ is the unit time parameter that trades off the running time of Perseus's optimizer and the granularity of energy schedules discovered by Perseus.}
This is repeated until iteration time reaches $T_{\min}$.

We note that tracing down from the energy schedule that consumes the maximum energy (\ie, Figure~\ref{fig:design-frontier-navigation} black dot) would be incorrect.
That schedule is far from optimized because, although it will execute with the least amount of time, stage imbalance leaves room for energy reduction (\S\ref{sec:motivation-intrinsic}).

\begin{algorithm}[t]
  \algrule{}

  \KwIn{DAG $\mathcal{G}$ of computations $i \in \mathcal{G}$\newline
        Amount of time to reduce in one iteration $\tau$\newline
        Iteration time with all max frequencies $T_{\min}$
  }
  \KwOut{Set of \revised{optimized} schedules $\mathcal{S}$}

  \algrule{}

  \Comment{Begin with the minimum energy schedule}
  $s \leftarrow$ Minimum energy for all computations\;\label{algoline:design-min-energy-plan}
  $\mathcal{S} \leftarrow \{ s \}$\;

  \While{$\mathrm{IterationTime}(\mathcal{G}, s) > T_{\min}$}{
    \Comment{Reduce time by $\tau$ with minimal energy increase (\S\ref{sec:design-reducing-time-optimally})}
    $s \leftarrow$ GetNextSchedule($\mathcal{G}$, $s$, $\tau$)\label{algoline:design-reducing-time-optimally}\;
    $\mathcal{S} \leftarrow \mathcal{S} \cup \{ s \}$ \;
  }

  \Return{$\mathcal{S}$}

  \algrule{}

  \caption{Iteratively discovering the frontier.}\label{algo:design-overview}
\end{algorithm}

Algorithm~\ref{algo:design-overview} provides an overview of our optimization process.
First, the energy schedule with the minimum energy consumption is constructed by planning every computation to run with minimum energy (line~\ref{algoline:design-min-energy-plan}).
Starting from there, the iteration time of the schedule is iteratively reduced by unit time $\tau$ \emph{while incurring minimal energy increase} (line~\ref{algoline:design-reducing-time-optimally}; Section~\ref{sec:design-reducing-time-optimally}).
This is repeated until the total iteration time of the schedule can no longer be reduced, and every energy schedule encountered in the process forms our frontier.

\subsection{Finding the Neighbor Energy Schedule}\label{sec:design-reducing-time-optimally}

\begin{figure*}[t!]
  \centering
  \includegraphics[width=0.85\textwidth]{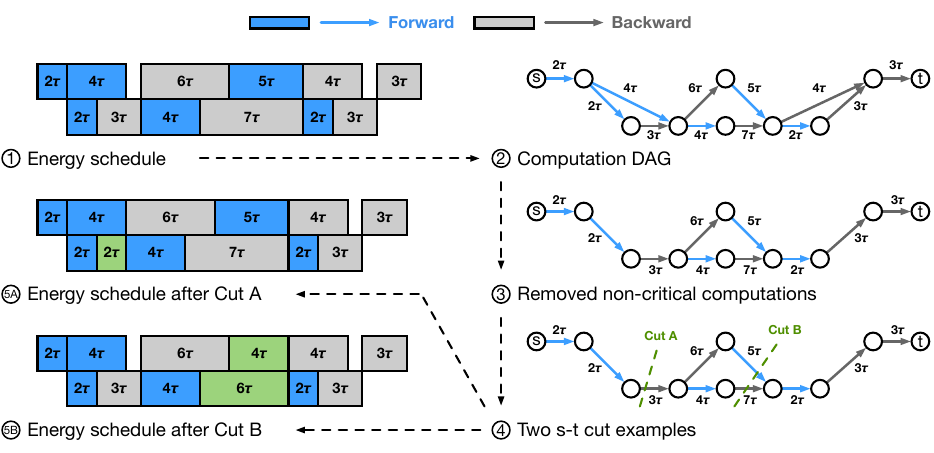}
  \vspace{-3mm}
  \caption{A simplified example of how to reduce iteration time by unit time $\tau$.
Given a 1F1B pipeline schedule with 2 stages and 3 microbatches ({\normalfont\protect\circled{1}}), it is first transformed to an equivalent representation of computation DAG ({\normalfont\protect\circled{2}}). 
Then the Critical DAG ({\normalfont\protect\circled{3}}) is obtained by \revised{considering only} the computations on the critical path.
Our key observation is that any valid s-t cut on the Critical DAG will reduce the iteration time by unit time $\tau$.
Cut A and Cut B are two examples of valid s-t \revised{cuts} ({\normalfont\protect\circled{4}}).
Either reducing the one computation associated with Cut A  ({\scalebox{0.8}{\small\protect\circled{5A}}}) or reducing the two computations associated with Cut B  ({\scalebox{0.8}{\small\protect\circled{5B}}}) reduces the iteration time by $\tau$.
}\label{fig:design-dags}
\end{figure*}

In this section, we describe our core subroutine GetNextSchedule (Algorithm~\ref{algo:design-overview}, line~\ref{algoline:design-reducing-time-optimally}).
Figure~\ref{fig:design-dags} provides visualizations of the process.
The entire procedure is given in Algorithm~\ref{algo:design-reduce-time-optimally}.

\paragraph{Node- and Edge-Centric Computation DAGs.}
Originally, Perseus's representation of the computation DAG is node-centric, which has forward and backward computations as nodes and their dependencies as edges.
As a setup for subsequent steps, we convert this into an edge-centric computation DAG where computations are edges and dependencies are nodes (\ie, all incoming edges must complete before any outgoing edge can begin).
This conversion can be done by splitting each node into two and connecting the two with an edge annotated with the computation on the original node.

\paragraph{Removing Non-Critical Computations.}
\looseness=-1
Our goal is to reduce the execution time of the computation DAG by $\tau$, which is equivalent to reducing the length of \emph{all critical paths} by $\tau$.\footnote{Let's say there are two critical paths that run in parallel. They must be of equal length to both be critical paths. Here, if only one were shortened, the other will remain the sole critical path and the DAG will not execute faster.}
Since computations that are not on any critical path (\ie, non-critical computations) do not affect the length of the critical path, we remove them from the computation DAG.\@

\paragraph{Finding Computations to Speed Up.}
Which computations on the DAG should we speed up in order to reduce the length of all critical paths by $\tau$?
The key observation is that \emph{any} s-t cut on the computation DAG represents a way to reduce the execution time of the DAG by $\tau$.
Specifically, by speeding up the computations on all cut edges by $\tau$, the entire computation DAG can be sped up exactly by $\tau$.

Figure~\ref{fig:design-dags} shows two examples of this.
\circled{4} shows two valid s-t cuts: \emph{Cut A} and \emph{Cut B}.
\circled{\footnotesize{5A}} speeds up the computation edge cut by \emph{Cut A} from $3\tau$ to $2\tau$, and the iteration time of the energy schedule was reduced by $\tau$.
Similarly, \circled{\footnotesize{5B}} speeds up the computation edges cut by \emph{Cut B} from $5\tau$ to $4\tau$ and from $7\tau$ to $6\tau$, and the iteration time of the energy schedule was also reduced by $\tau$.
Especially, in the second case, iteration time was only reduced because computations on two parallel critical paths were sped up \emph{together}.

\paragraph{Solving with Minimum Cut.}
We have seen that any s-t cut represents a way to speed up the entire DAG by $\tau$.
But speeding up computations increases energy.
Then, a natural question is, which cut brings the smallest energy increase?

\revised{We can precisely map the \emph{flow capacity} of an s-t cut to the amount of \emph{energy increase} from speeding up cut edges.
That is, by finding the amount of energy increase each computation will incur with the slope of its exponential function (\S\ref{sec:design-problem-formulation}) and defining it to be the edge's flow capacity, we can reduce our problem to minimum cut, which we can solve with maximum flow.}
After finding the minimum cut, we modify the durations of the computations involved in the cut, obtaining the neighbor energy schedule.
Appendix~\ref{sec:appendix-min-cut} provides details with mathematical expressions for flow capacities.

\paragraph{Converting Back to GPU Frequencies.}
\looseness=-1
Finally, we convert the energy schedule into GPU frequencies that can be realized by the Perseus client.
For each computation, we convert its planned execution time $t$ to the slowest GPU frequency that will execute \emph{faster} than $t$. %
This is because when computations are tightly packed by our algorithm, while slightly speeding up a computation is acceptable, slowing down \emph{any computation} on the critical path will directly slow down the entire DAG, increasing intrinsic energy bloat.

\begin{table*}[t]
  \centering
  \begin{tabular}{lp{10cm}}
      \toprule
      \textbf{API} & \textbf{Description} \\
      \midrule
      \texttt{profiler.begin(type)} & Begin time and energy profiling for computation \texttt{type}. \\
      \texttt{profiler.end(type)} & Record time and energy profiling results for computation \texttt{type}. \\
      \texttt{controller.set\_speed(type)} & Set the hardware's execution speed as planned for computation \texttt{type}. \\
      \texttt{server.set\_straggler(id, delay, degree)} & Notify that a straggler is anticipated after \texttt{delay} seconds. A straggler returning to normal can be communicated by setting \texttt{degree} to 1. \\
      \bottomrule
  \end{tabular}
  \caption{The minimal set of Perseus client and server APIs that require implementation. One client process manages each accelerator. The \texttt{type} parameter should be either \texttt{"forward"} or \texttt{"backward"}. On GPUs, the ``speed'' control knob is the SM frequency. \texttt{set\_straggler} is invoked by the infrastructure with the \texttt{id} of an accelerator to notify the server via HTTP/RPC.}\label{tab:implementation-client-api}
\end{table*}

\paragraph{Time Complexity Analysis.}
Our optimization algorithm has polynomial runtime. 
Let $N$ and $M$ denote the number of stages and microbatches, respectively.
Then, the computation DAG will have $O(NM)$ number of nodes and edges, and maximum flow with Edmonds-Karp runs in $O(N^3M^3)$.
While for general DAGs the total number of steps is known to be exponential to the size of the DAG~\cite{skutella-thesis}, we prove that for DAGs that represent pipeline computations, the number of steps is $O(N + M)$, yielding a final polynomial time complexity of $O((N + M)N^3M^3)$.
See Appendix~\ref{sec:appendix-polytime} for proof.

In reality, commonly used number of stages ($N$) is 4 to 8 (at most tens) to reduce pipeline bubble ratio~\cite{megatronlm-sc21,mlperf_training_31}.
Number of microbatches ($M$) is typically around $4 N$~\cite{gpipe-neurips19,megatron-turing-arxiv22}, but \revised{recently with} high data parallel degree, far \revised{fewer} have been reported even for high-performance settings~\cite{mlperf_training_31}.
As such, algorithm runtime is practically negligible (\S\ref{sec:eval-overhead}), especially given that large model training easily takes weeks or months~\cite{patterson2021carbon}.

\begin{algorithm}[t!]
  \algrule{}

  \KwIn{DAG $\mathcal{G}$ of computations $i \in \mathcal{G}$\newline
        Current energy schedule $s$\newline
        Amount of iteration time to reduce $\tau$
  }
  \KwOut{Neighbor schedule with reduced time $s'$}

  \algrule{}

    \Comment{Construct edge-centric computation DAG (\circled{2})}
    $\mathcal{G} \leftarrow$ Split nodes into two and connect with edge\;

    \Comment{Find and remove non-critical computations (\circled{3})}
    Annotate earliest \& latest start times for $\forall i \in \mathcal{G}$\;
    \For{$i \in \mathcal{G}$}{
      \If{$i$ \upshape has different earliest and latest start}{
        Remove $i$ from $\mathcal{G}$\;
      }
    }

    \Comment{Find set of computations to modify (\circled{4})}
    $S$, $T \leftarrow$ FindMinCut($\mathcal{G}$, $s$)\;

    \Comment{Modify computation durations (\circled{5})}
    Modify duration of $\forall i$ in $S - T$ cut by $\tau$ \;

    \Comment{Assign frequencies from planned computation times}
    $s' \leftarrow$ $\min f_i$ that runs no slower than planned\;

    \Return{$s'$}

  \algrule{}

  \caption{GetNextSchedule: Reducing the execution time of the DAG by $\tau$ with minimal energy increase.}\label{algo:design-reduce-time-optimally}
\end{algorithm}

\subsection{Generalizations}\label{sec:design-extensions}

In this section, we present generalizations to our optimization algorithm useful for planning large model training.

\paragraph{3D/Hybrid Parallelism.}
Operator parallelism techniques (e.g., data, tensor, or sequence parallelism) split operations in \emph{equal sizes}, resulting in each GPU running the same computation.
This allows Perseus to profile only one GPU per stage, decide the energy schedule for that GPU, and replicate it to all other GPUs in the same stage.
We show that Perseus works well for 3D parallelism in Section~\ref{sec:eval-frontier}.

\paragraph{Constant-Time Operations.}
\revised{There are operations in the training pipeline that may take non-trivial latency, other than computation and blocking on communication.}
For instance, loading and copying input data into VRAM or communication over slower links can take considerable latency.
However, the time and energy consumption of these operations are not affected by the GPU's frequency.
Perseus can take constant-time operations into account during planning by viewing them as a node with only one frequency choice.

\paragraph{Other Pipeline Schedules.}
There are various schedules for pipeline parallel training, including GPipe~\cite{gpipe-neurips19}, 1F1B~\cite{pipedreamflush-icml21}, interleaved 1F1B~\cite{megatronlm-arxiv}, and early recomputation 1F1B~\cite{merak-tpds23}.
As long as the computations on the schedule can be expressed as a DAG, Perseus can optimize its energy consumption without modification.
As long as there is stage imbalance, any pipeline schedule will have intrinsic energy bloat.

\section{Implementation}\label{sec:implementation}

The Perseus server and client are implemented in Python.
Perseus can optimize any training infrastructure, framework, and accelerator as long as the APIs in Table~\ref{tab:implementation-client-api} can be implemented, and the accelerator supports multiple execution speeds that trade off computation time and energy.

As a reference, we have integrated the Perseus client with Merak~\cite{merak-tpds23}, which marries high-performance tensor parallelism of Megatron-LM~\cite{megatronlm-github} and the generic pipeline execution engine of DeepSpeed~\cite{deepspeed-github}.
While training engine implementations differ widely, many have separate code blocks for forward and backward, allowing them to be wrapped with the \texttt{profiler} APIs.
We provide an example of what is looks like to integrate the client with a training engine in Appendix~\ref{sec:appendix-integration}.
Activation recomputation~\cite{checkpointing-arxiv16} is enabled to allow large batch sizes to fit in GPUs.

\paragraph{Profiler.}
Accurate profiling is important to our optimization algorithm; inaccurate latency profiles (especially underestimations) may slow down the end-to-end latency of the DAG, whereas inaccurate energy profiles can lead the algorithm to incorrectly select computations to speed up.

Fortunately, the latency of a fixed set of GPU computations, especially with the GPU's frequency locked, is known to be very stable~\cite{nvml,clockwork-osdi20}.
Furthermore, to ensure that profiling results are representative of real training, the Perseus client profiles the time and energy of each forward and backward computation at the beginning of the training job \emph{in vivo}.
Each supported GPU frequency is profiled one by one from the highest to the lowest for about five iterations (more if one iteration has less microbatches).
After a certain frequency, lower frequencies result in both more time \emph{and} energy consumed, making them strictly suboptimal compared to higher frequencies.
Profiling is terminated at that point.

Finally, we profile $P_{\textrm{blocking}}$ using two GPUs.
One GPU blocks on P2P communication and the other sleeps, and we measure the power consumption of the blocking GPU.\@
It is sufficient to profile $P_{\textrm{blocking}}$ \emph{once per GPU model}.

\paragraph{Asynchronous Frequency Controller.}
The client-side controller spawns a separate process that asynchronously sets the GPU's frequency through NVML~\cite{nvml} without blocking the main training process.
Training frameworks can call \texttt{set\_speed} at the beginning of forward or backward to set the GPU's frequency as planned by the server.

\section{Evaluation}\label{sec:eval}

We evaluate Perseus on five workloads and compare it against EnvPipe and Zeus. 
Our key findings are the following:

\begin{denseitemize}
  \item Perseus can effectively reduce intrinsic and extrinsic energy bloat.
  Training on real GPUs shows up to 28.5\% energy savings using Perseus (\S\ref{sec:eval-reducing-energy-bloat}).
  
  \item In emulated large-scale training scenarios, Perseus significantly outperforms the baselines by consistently providing up to 30\% energy savings (\S\ref{sec:eval-large-scale-emulation}).

  \item Energy bloat reduction \revised{is} possible because Perseus can enumerate efficient energy schedules on the \revised{time--energy} frontier (\S\ref{sec:eval-frontier}).

  \item Perseus reduces energy bloat with low overhead (\S\ref{sec:eval-overhead}).
\end{denseitemize}

\subsection{Experimental Setup}\label{sec:eval-exp-setup}

\paragraph{Testbed.}
We run our evaluation workloads in a GPU cluster, where each node is equipped with an AMD EPYC 7513 CPU, 512 GB DRAM, and four NVIDIA A40-48G GPUs.\@
For A100 results, we use a node provided by Chameleon Cloud~\cite{chameleoncloud}, equipped with two Intel Xeon Platinum 8380 CPUs, 512 GB DRAM, and four NVIDIA A100-80G PCIe GPUs.

\paragraph{Workloads and experiment parameters.}
We evaluate Perseus with various workloads spanning from GPT-3~\cite{gpt3}, Bloom~\cite{bloom-arxiv22}, BERT~\cite{bert}, T5~\cite{t5-jmlr20}, to Wide-ResNet~\cite{wideresnet-bmvc16}.
We use model variants with 1.3B to 6.7B parameters to run the models in our testbed, and scale them up to 176B parameters in large-scale emulation.
We chose the microbatch size and number of microbatches that yield the highest throughput given the global batch size.
We use the minimum imbalance stage partitioning method described in Section~\ref{sec:motivation-intrinsic} for all workloads.
Appendix~\ref{sec:appendix-workloads} lists complete model configurations, parameters, and stage partitioning details.

\paragraph{Metrics.}
We report GPU energy reduction and slowdown of a training iteration (\%) relative to using all maximum GPU frequencies.
In most cases slowdown is close to zero, in which case energy and average power reductions coincide.
Reducing only extrinsic bloat is not possible, because Perseus reduces extrinsic bloat \emph{while keeping} intrinsic bloat low as it slows down non-straggler pipelines.
Therefore, we report
(1) intrinsic bloat reduction \emph{without} stragglers and
(2) intrinsic + extrinsic bloat reduction \emph{with} stragglers.

\paragraph{Baselines.}
We mainly compare with two prior works:
\begin{denseitemize}
  \item \textbf{EnvPipe}~\cite{envpipe-atc23} reduces only intrinsic energy bloat while trying to minimize slowdown.
  We compare Perseus's energy bloat reduction with EnvPipe (\S\ref{sec:eval-reducing-energy-bloat}, \S\ref{sec:eval-large-scale-emulation}).
  \item \textbf{Zeus}~\cite{zeus-nsdi23} characterizes the time--energy tradeoff of single GPU training.
    We compare Perseus's \revised{time--energy frontier} against that of Zeus (\S\ref{sec:eval-frontier}).
\end{denseitemize}

\subsection{Reducing Energy Bloat}\label{sec:eval-reducing-energy-bloat}

We start with overall energy bloat \revised{reduction---intrinsic bloat without stragglers (\S\ref{sec:eval-intrinsic}) and intrinsic + extrinsic bloat with stragglers (\S\ref{sec:eval-intrinsic-extrinsic})---achieved} by Perseus and EnvPipe.
All numbers were obtained by running on testbed GPUs.
All solutions use the same amount of GPU hardware resources.

\subsubsection{Intrinsic Bloat Reduction Without Stragglers}\label{sec:eval-intrinsic}

\begingroup
\setlength{\tabcolsep}{4pt}

\begin{table}[t]
  \footnotesize
  \centering
  \begin{subtable}[c]{0.45\textwidth}
  \centering
  \begin{tabular}{lrrrr}
    \toprule
    \multirow{2}{*}{\textbf{Model}} & \multicolumn{2}{c}{\textbf{Energy Savings (\%)}} & \multicolumn{2}{c}{\textbf{Slowdown (\%)}} \\
                   & Perseus & EnvPipe & Perseus & EnvPipe \\
    \midrule
    GPT-3 1.3B        & 13.2 & 8.8 & 0.1 & 0.1 \\
    BERT 1.3B         & 12.9 & 8.0 & 0.5 & 0.0 \\
    T5 3B             & 10.6 & 7.4 & 1.3 & 3.4 \\
    Bloom 3B          & 11.7 & 8.9 & 0.2 & 0.2 \\
    Wide-ResNet 1.5B  & 3.2  & 3.7 & 2.3 & 4.1 \\
    \bottomrule
  \end{tabular}
  \subcaption{Four stage pipeline parallelism on A100 GPUs}
  \end{subtable}

  \bigskip

  \begin{subtable}[c]{0.45\textwidth}
    \centering
    \begin{tabular}{lrrrr}
      \toprule
      \multirow{2}{*}{\textbf{Model}} & \multicolumn{2}{c}{\textbf{Energy Savings (\%)}} & \multicolumn{2}{c}{\textbf{Slowdown (\%)}} \\
                     & Perseus & EnvPipe & Perseus & EnvPipe \\
      \midrule
      GPT-3 2.7B       & 21.1 & 21.7 & 0.2 & 5.6 \\
      BERT 1.3B        & 15.7 & 16.5 & 0.0 & 9.7 \\
      T5 3B            & 28.5 & 19.3 & 0.0 & 0.0 \\
      Bloom 3B         & 22.4 & 19.9 & 0.0 & 0.0 \\
      Wide-ResNet 1.5B & 20.4 & 16.5 & 0.2 & 0.5 \\
      \bottomrule
    \end{tabular}
    \subcaption{Eight stage pipeline parallelism on A40 GPUs}
  \end{subtable}

\caption{[Experiment] Intrinsic energy bloat (without stragglers) reduction and iteration time slowdown.}\label{tab:eval-frontier-min-slowdown}
\end{table}

\endgroup

Table~\ref{tab:eval-frontier-min-slowdown} compares the energy savings achieved by Perseus's minimum iteration time energy schedule (leftmost point of the \revised{time--energy} frontier) and that by EnvPipe.

We make two observations regarding Perseus.
First, models show varying amounts of energy savings because
(1) their stage \revised{imbalances} vary (Table~\ref{tab:motivation-imbalance}), and
(2) their forward and backward are composed of different computations, which affects time/energy sensitivity when changing the frequency.
For instance, unlike other models, Wide-ResNet 1.5B on A100 after minimum imbalance stage partitioning has nearly perfect stage balance, resulting in low intrinsic energy bloat.
However, as will be seen in Section~\ref{sec:eval-intrinsic-extrinsic}, such models tend to achieve greater extrinsic bloat savings because most of their computations run at a high frequency, and slowing them down due to stragglers leads to higher energy reduction.

Second, A40 demonstrates more energy savings compared to A100.
This is because the dynamic clock frequency range of A100 (210--1410 MHz) is smaller than that of A40 (210--1740 MHz).
Thus, tuning down the GPU's frequency yields a relatively smaller change in computation time and energy compared to those at the maximum frequency.
However, we expect the more recent GPUs to have better percentage savings due to higher maximum frequency (e.g., 1980 MHz for H100 SXM~\cite{H100whitepaper}) and better absolute savings due to high TDP (e.g., 1,200 W for each GPU on GB200~\cite{blackwell-news}).

EnvPipe in general provides lower energy savings, primarily due to its assumption that the final stage of a pipeline is always the heaviest.
This is only correct with a probability of $1/N$, where $N$ is the number of pipeline stages.
Additionally, it sometimes considerably degrades iteration time because it is not aware of \emph{single-choice operations} inside the pipeline (\S\ref{sec:design-extensions}) and can slow down some computations too much.

\subsubsection{Intrinsic + Extrinsic Bloat Reduction With Stragglers}\label{sec:eval-intrinsic-extrinsic}

\begin{table}[t]
  \footnotesize
  \centering
  \begin{subtable}[c]{0.45\textwidth}
    \centering
    \begin{tabular}{ll|r@{\hspace{1em}}r@{\hspace{1em}}r@{\hspace{1em}}r@{\hspace{1em}}r@{\hspace{1em}}r}
      \toprule
      \textbf{Model}      & \multirow{2}{*}{\textbf{Method}} & \multicolumn{6}{c}{\textbf{Energy Savings (\%) given $T' / T$}} \\
      \textbf{\# Params}  &      & 1.05 & 1.1 & 1.2 & 1.3 & 1.4 & 1.5 \\
      \midrule
      GPT-3     & Perseus  & 14.7 & 15.9 & 15.5 & 15.0 & 14.6 & 14.3 \\
      1.3B      & EnvPipe  & 8.7  & 8.5  & 8.3  & 8.1  & 7.9  & 7.7  \\
      \midrule
      Bloom     & Perseus & 13.6 & 15.6 & 15.2 & 14.7 & 14.3 & 14.0 \\
      3B        & EnvPipe &  8.8  & 8.7  & 8.4  & 8.2  & 8.0  & 7.8 \\
      \midrule
      BERT      & Perseus  & 14.9 & 16.9 & 16.4 & 15.9 & 15.5 & 15.0 \\
      1.3B      &  EnvPipe & 7.9  & 7.8  & 7.5  & 7.3  & 7.1  & 6.9 \\
      \midrule
      T5        & Perseus  & 15.3 & 18.0 & 17.9 & 17.4 & 16.9 & 16.5 \\
      3B        & EnvPipe  & 8.4  & 8.2  & 8.0  & 7.8  & 7.6  & 7.4 \\
      \midrule
      Wide-ResNet & Perseus & 9.4  & 12.7 & 12.6 & 12.3 & 12.0 & 11.6 \\
      1.5B        & EnvPipe & 4.9  & 4.8  & 4.7  & 4.5  & 4.4  & 4.3 \\
      \bottomrule
    \end{tabular}
    \subcaption{Four stage pipeline parallelism on A100 GPUs}
  \end{subtable}

  \bigskip

  \begin{subtable}[c]{0.45\textwidth}
    \centering
    \begin{tabular}{ll|r@{\hspace{1em}}r@{\hspace{1em}}r@{\hspace{1em}}r@{\hspace{1em}}r@{\hspace{1em}}r}
      \toprule
      \textbf{Model}     & \multirow{2}{*}{\textbf{Method}} & \multicolumn{6}{c}{\textbf{Energy Savings (\%) given $T' / T$}} \\
      \textbf{\# Params} &  & 1.05 & 1.1 & 1.2 & 1.3 & 1.4 & 1.5 \\
      \midrule
      GPT-3       & Perseus & 24.5 & 26.0 & 25.9 & 25.2 & 24.6 & 24.0 \\
      2.7B        & EnvPipe & 22.9 & 22.6 & 22.0 & 21.4 & 20.9 & 20.4 \\
      \midrule
      Bloom       & Perseus & 25.5 & 26.4 & 25.9 & 25.2 & 24.6 & 24.0 \\
      3B          & EnvPipe & 19.6 & 19.3 & 18.8 & 18.3 & 17.8 & 17.4 \\
      \midrule
      BERT        & Perseus & 20.0 & 22.6 & 24.1 & 23.4 & 22.8 & 22.2 \\
      1.3B        & EnvPipe & 19.2 & 18.9 & 18.3 & 17.8 & 17.4 & 16.9 \\
      \midrule
      T5          & Perseus & 27.9 & 27.3 & 26.2 & 25.2 & 24.3 & 23.4 \\
      3B          & EnvPipe & 18.4 & 18.0 & 17.3 & 16.6 & 16.0 & 15.4 \\
      \midrule
      Wide-ResNet & Perseus & 24.3 & 26.2 & 26.3 & 25.7 & 25.0 & 24.4 \\
      1.5B        & EnvPipe & 16.4 & 16.2 & 15.8 & 15.4 & 15.0 & 14.6 \\
    \bottomrule
  \end{tabular}
    \subcaption{Eight stage pipeline parallelism on A40 GPUs}
  \end{subtable}

\caption{[Experiment] Energy savings given varying straggler slowdown ($T' / T$).
Perseus can reduce extrinsic bloat while keeping intrinsic bloat low, whereas EnvPipe cannot.}\label{tab:eval-frontier-varying-slowdown}
\end{table}

\looseness=-1
When stragglers create extrinsic energy bloat, the amount of energy savings depends on how much energy reduction the \revised{time--energy} frontier yields for longer iteration times.
Table~\ref{tab:eval-frontier-varying-slowdown} shows the amount of energy savings of a non-straggler pipeline given varying straggler slowdowns.
For a given slowdown factor ($T' / T$), a non-straggler pipeline's iteration time is set to be $T_\textrm{opt} = \min(T^*, T')$ (\S\ref{sec:overview-optimization}), and energy reduction comes from (1) slowing down the pipeline itself and (2) reducing the time (and energy) blocking on communication, waiting for the straggler.

The percentage of savings \revised{initially} increases with the straggler's iteration time, but then \revised{gradually decreases} as it slows down beyond $T^*$.
This is expected.
The absolute amount of energy reduction in Joules is the largest when the straggler's iteration time is $T^*$ and constant afterward, because Perseus does not slow down non-straggler pipelines beyond $T^*$ (\S\ref{sec:overview-optimization}).
Thus, as the straggler slows down beyond $T^*$, additional time and energy is consumed while waiting for the straggler, lowering the \emph{percentage} of energy savings.

Finally, the point of maximum energy savings is different for each model.
This is because each model has a different $T^*$ value, which is determined by how much each stage's computation slows down on the minimum-energy frequency.

\subsubsection{How Much Potential Saving Was Realized?}

The largest possible savings under our problem setting occurs when running every computation at their minimum-energy frequencies (\ie, the $T^*$ point on the \revised{time--energy} frontier).
For intrinsic bloat without stragglers, Perseus \revised{realizes on average} 74\% and 89\% \revised{of the potential savings} on A100 and A40, respectively, with negligible slowdown.
This is possible because there are much more non-critical computations in the DAG that can be slowed down than critical ones.
With stragglers, Perseus fully realizes potential savings when the straggler's slowdown degrees are on average 1.1 and 1.15 on A100 and A40 respectively, which is not unrealistic considering slowdowns reported in literature (\S\ref{sec:motivation-extrinsic}).

\subsection{Large-Scale Emulation}\label{sec:eval-large-scale-emulation}

Because we do not have access to a GPU cluster required to run huge models like GPT-3 175B, we use \emph{emulation} grounded on fine-grained profiling for large-scale evaluation.
In general, trends in our emulation result match those obtained from real training in Section~\ref{sec:eval-reducing-energy-bloat}.

\begin{table}[t]
  \footnotesize
  \centering
  \begin{tabular}{cccc}
    \toprule
    \textbf{\begin{tabular}[c]{@{}c@{}}\# GPUs\end{tabular}} & \textbf{\# Pipelines} & \textbf{\begin{tabular}[c]{@{}c@{}}\# Microbatches\\ Per Pipeline\end{tabular}} & \textbf{\begin{tabular}[c]{@{}c@{}}Global\\ Batch Size\end{tabular}} \\
    \midrule
    1024                   & 16                    & 96                                                                               & \multirow{4}{*}{1536}            \\
    2048                   & 32                    & 48                                                                               &                                  \\
    4096                   & 64                    & 24                                                                               &                                  \\
    8192                   & 128                   & 12                                                                               &                                  \\
    \bottomrule
  \end{tabular}
  \caption{Strong scaling parameters for large-scale emulation.
  A pipeline has tensor parallel degree 8 and 8 pipeline stages.}\label{tab:eval-emulation-strong-scaling-parameters}
\end{table}

\paragraph{Emulation Methodology.}
We profile the time and energy consumption of each layer (e.g., Transformer decoder) in GPT-3 175B and Bloom 176B and run our optimization algorithm to obtain the \revised{time--energy} frontier.
We perform \emph{strong scaling} when varying the number of GPUs (Table~\ref{tab:eval-emulation-strong-scaling-parameters}) in order to keep the global batch size constant~\cite{largebatch1,largebatch2}.
We used A100 SXM GPUs for emulation, which we believe are more representative of large-scale training infrastructure. %

\paragraph{Emulator Fidelity.}
We compare the percentage of energy savings vs.\ running all maximum frequencies given by our emulator and real experiments for our A100 workloads and find that the emulator \emph{always underestimates} energy savings.
Specifically, savings on the leftmost and rightmost point of the frontier are underestimated by 18.6\% and 21.7\% on average, respectively.
We believe this is due to our simplifying assumption that $P_\mathrm{blocking}$ is constant regardless of the GPU's frequency.
This means \revised{the} savings given by the emulator can be considered a lower bound for actual savings.

\begin{figure}[t]
  \centering
  \subfloat{
    \includegraphics*[width=0.5\linewidth]{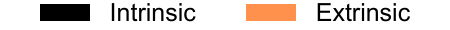}
  }

  \setcounter{subfigure}{0}

  \subfloat[A100]{
      \includegraphics[width=0.45\linewidth]{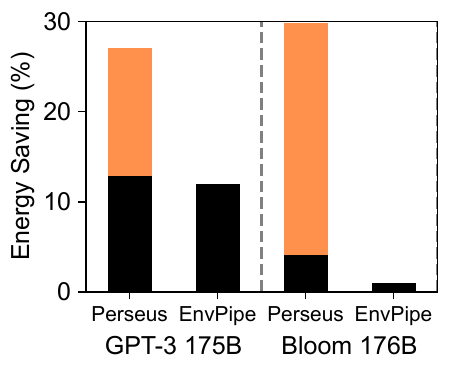}\label{fig:eval-bloat-overview-a100}
  }
  \subfloat[A40]{
      \includegraphics[width=0.45\linewidth]{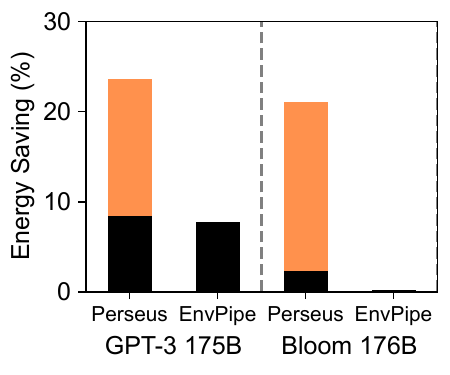}\label{fig:eval-bloat-overview-a40}
  }
  \vspace{-3mm}
  \caption{[Emulation] Energy savings breakdown with straggler slowdown 1.2 and 1,024 GPUs.}\label{fig:eval-large-model-bloat-e2e}
\end{figure}

\paragraph{Result Summary.}
Figure~\ref{fig:eval-large-model-bloat-e2e} breaks down the amount of energy bloat reduction for GPT-3 175B and Bloom 176B when slowdown degree is 1.2 on 1,024 GPUs.
EnvPipe can only reduce intrinsic bloat as it does not provide \revised{a time--energy} frontier; even for intrinsic bloat, it is suboptimal.
In contrast, Perseus reduces energy consumption by up to 30\% by reducing both intrinsic and extrinsic energy bloat.

\begin{table}[t]
  \footnotesize
  \centering
  \begin{tabular}{cc|rrrr}
    \toprule
    \multirow{2}{*}{\textbf{Model}} & \multirow{2}{*}{\textbf{GPU Type}} & \multicolumn{4}{c}{\textbf{\begin{tabular}[c]{@{}c@{}}Energy Savings (\%)\\ Given \# Microbatches\end{tabular}}} \\
                                    &                                    & 12     & 24    & 48    & 96    \\
    \midrule
    \multirow{2}{*}{GPT-3 175B}     & A100                               & 15.20  & 14.19 & 13.62 & 13.32 \\
                                    & A40                                & 11.81  & 10.22 & 9.34  & 8.88  \\
    \midrule
    \multirow{2}{*}{Bloom 176B}     & A100                               & 10.47  & 7.06  & 5.23  & 4.28  \\
                                    & A40                                & 6.97   & 4.49  & 3.12  & 2.41  \\
    \bottomrule
  \end{tabular}
  \caption{[Emulation] Perseus's intrinsic energy bloat reduction without stragglers for GPT-3 175B and Bloom 176B.
  Number of microbatches is varied following Table~\ref{tab:eval-emulation-strong-scaling-parameters}.}\label{tab:eval-intrinsic-bloat-gpt3}
\end{table}

\paragraph{Intrinsic Bloat Reduction Without Stragglers.}
Table~\ref{tab:eval-intrinsic-bloat-gpt3} shows Perseus's intrinsic energy bloat reduction without stragglers for GPT-3 175B and Bloom 176B.
The number of microbatches \revised{is} varied based on Table~\ref{tab:eval-emulation-strong-scaling-parameters}.
For all models, as \revised{more microbatches} are added to the pipeline, the amount of intrinsic bloat decreases.
This is fundamentally due to the ratio of microbatches in the 1F1B's warm-up and flush phase (beginning and end) \revised{versus the} steady state phase (middle).
Most microbatches in the warm-up and flush phases can slow down until their minimum energy frequency, yielding large energy savings.
However, microbatches in the pipeline's steady state cannot slow down to their full potential when the amount of stage imbalance is not large, thereby yielding modest savings.
When the number of microbatches in the pipeline increases, only the number of steady state microbatches increases, and energy reduction converges to the average energy savings of steady state microbatches.

\begin{figure}[t!]
  \centering
  \subfloat{}{
    \includegraphics[width=0.5\textwidth]{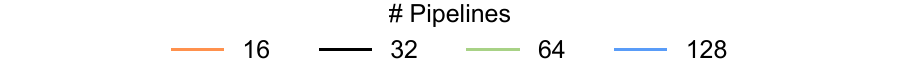}
  }
  \setcounter{subfigure}{0}
  \subfloat[GPT-3 175B on A100]{
      \includegraphics[width=0.21\textwidth]{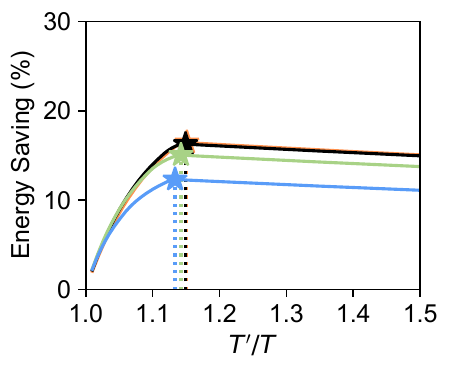}\label{fig:eval-extrinsic-gpt3-a100}
  }
  \subfloat[Bloom 176B on A100]{
      \includegraphics[width=0.21\textwidth]{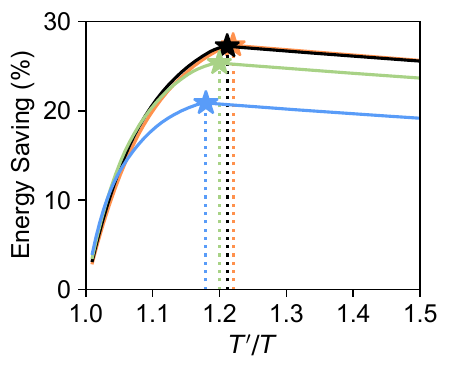}\label{fig:eval-extrinsic-bloom-a100}
  }
  \hfill %
  \subfloat[GPT-3 175B on A40]{
      \includegraphics[width=0.21\textwidth]{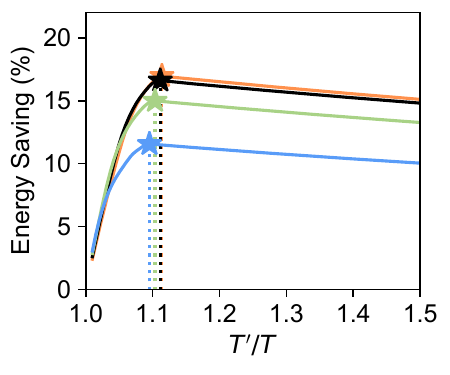}\label{fig:eval-extrinsic-gpt3-a40}
  }
  \subfloat[Bloom 176B on A40]{
      \includegraphics[width=0.21\textwidth]{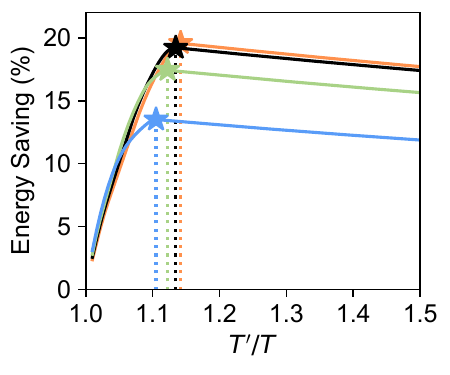}\label{fig:eval-extrinsic-bloom-a40}
  }
  \vspace{-3mm}
  \caption{[Emulation] Perseus's intrinsic + extrinsic energy bloat reduction with varying straggler slowdown ($T' / T$).
    Number of pipelines is varied following Table~\ref{tab:eval-emulation-strong-scaling-parameters}.
    $\bigstar$ denotes $T^*/T$ for each pipeline.
    Please note the different Y-axes.}\label{fig:eval-extrinsic-energy-bloat}
\end{figure}

\begin{figure*}[!t]
  \centering
  \subfloat{
    \includegraphics*[width=0.45\textwidth]{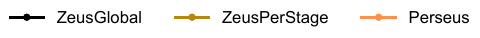}
  }

  \setcounter{subfigure}{0}
  \subfloat[PP=4 on A100, GPT-3 1.3B]{\includegraphics[width=0.315\textwidth]{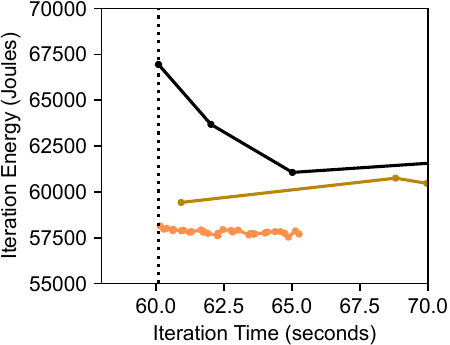}\label{fig:eval-gpt3-a100-pp4-frontier}}
  \hspace{4mm}
  \subfloat[PP=8 on A40, GPT-3 2.7B]{\includegraphics[width=0.315\textwidth]{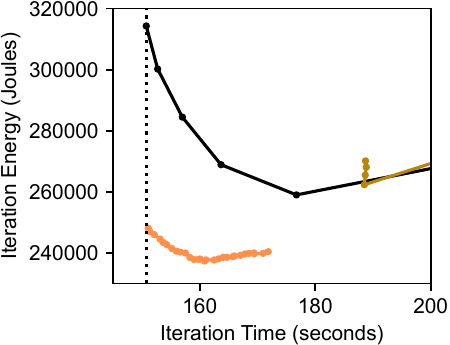}\label{fig:eval-gpt3-a40-pp8-frontier}}
  \hspace{4mm}
  \subfloat[DP=2, TP=2, PP=4 on A40, GPT-3 6.7B]{\includegraphics[width=0.315\textwidth]{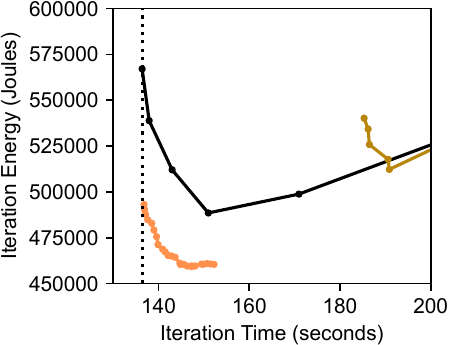}\label{fig:eval-gpt3-a40-3d-frontier}}
  \vspace{-3mm}
  \caption{[Experiment] Iteration time--energy frontiers for GPT-3, achieved by Perseus and the two baselines derived from Zeus~\cite{zeus-nsdi23}. Perseus Pareto-dominates all other approaches. The dotted vertical line is the iteration time of running all GPUs at their maximum power limit, which is the default mode of operation. 
  Please note the different X- and Y-axes.}\label{fig:eval-gpt3-frontiers}
\end{figure*}

\paragraph{Intrinsic + Extrinsic Bloat Reduction With Stragglers.}
We introduce stragglers of varying slowdowns in large-scale emulation. %
Figure~\ref{fig:eval-extrinsic-energy-bloat} reports the amount of intrinsic + extrinsic energy bloat reduction achieved by Perseus.
The trend where energy saving increases until $T' < T^*$ and wanes afterward is consistent with Section~\ref{sec:eval-intrinsic-extrinsic}.

An interesting observation here is that there is a tradeoff between scale and energy savings: \emph{configurations with more pipelines have less percentage of energy savings or less amount of energy savings per pipeline}.
It may seem intuitive to assume that more pipelines \revised{bring} more energy savings, as there is only one straggler pipeline and all the other pipelines can reduce their energy consumption.
However, this holds only in weak scaling scenarios, where the per-pipeline batch size is held constant (\ie, increasing the global batch size proportionally with the number of pipelines).
Instead, in the more realistic strong scaling configuration (\ie, the global batch size is constant and per-pipeline batch size is decreased as more pipelines deployed), each pipeline's number of microbatches changes.
With fewer microbatches, the ratio of pipeline bubble (time that GPUs are idle) at the beginning and end of each pipeline iteration increases~\cite{megatronlm-sc21}.
These bubbles cannot be perfectly eliminated by intrinsic energy bloat reduction, resulting in a smaller energy \revised{savings} percentages.
However, as the absolute number of GPUs increases, even a small savings percentage is expected to yield \emph{huge absolute energy savings}.

\subsection{Iteration Time--Energy Frontier Comparison}\label{sec:eval-frontier}

The energy bloat reductions in Sections~\ref{sec:eval-reducing-energy-bloat} and~\ref{sec:eval-large-scale-emulation} were made possible by the \revised{time--energy} frontier obtained using Perseus's optimization algorithm (\S\ref{sec:design}).
Here, we further examine the frontier with different parallelization configurations and models and compare against Zeus~\cite{zeus-nsdi23}.
Since Zeus only produces the training time--energy frontier for \emph{single-GPU} training jobs, we implemented two Zeus-based baselines for large model training scenarios:
\begin{denseitemize}
  \item \textbf{ZeusGlobal}: Scans one global power limit for all stages.
  \item \textbf{ZeusPerStage}: Finds a set of per-stage power limits that balances forward computation time.
\end{denseitemize}

Figure~\ref{fig:eval-gpt3-frontiers} shows the frontiers of all solutions for different sizes of GPT-3 under three parallelization configurations:
(a) four stage pipeline parallelism on A100;
(b) eight stage pipeline parallelism on A40; and 
(c) 3D parallelism (data parallelism 2, tensor parallelism 2, pipeline parallelism 4) on A40.
All results were obtained from running on testbed GPUs.
Appendix~\ref{sec:appendix-all-frontiers} has frontiers for other models.

\looseness=-1
Perseus Pareto-dominates both baselines derived from Zeus.
ZeusGlobal is unaware of pipeline stage imbalances and slows down every stage, and therefore is unable to intrinsic energy bloat.
ZeusPerStage can balance the forward computation time of each stage, but is unaware of the \emph{critical path} of the DAG, slowing down critical computations.
In contrast, Perseus can precisely slow down non-critical computations, effectively reducing energy bloat.

\subsection{Overhead of Perseus}\label{sec:eval-overhead}

\paragraph{Profiling.}
\revised{At} the beginning of training, the client-side profiler (\S\ref{sec:implementation}) profiles forward and backward computations in each stage.
For our A100 workloads, the initial profiling phase increased end-to-end training time by 13 minutes on average, which is negligible overhead for large model training.

\paragraph{Algorithm Runtime.}
The average runtime of the optimization algorithm (\S\ref{sec:design}) across our A100 workloads was 6.5 minutes, with the longest being Bloom 3B (15.7 minutes).
For our \revised{largest-scale} emulation workload (GPT-3 175B on A100 with 8,192 GPUs), the algorithm took 87 seconds, which is short because it is sufficient to optimize just one data parallel pipeline (\S\ref{sec:design-extensions}).
While runtime will increase with larger DAGs for larger models, we believe the overhead is justified because training time is likely to also increase with the scale of the training job.
Looking up the optimal energy schedule given the straggler's iteration time $T'$ is \revised{instantaneous}.

\section{Related \revised{Work}}\label{sec:related}

\paragraph{Large Model Training.}
Many recent works focus on enabling and accelerating large model training using 3D parallelism.
GPipe~\cite{gpipe-neurips19} and PipeDream~\cite{pipedream-sosp19} first introduced pipeline parallelism.
3D parallelism, especially popularized for Transformer-based models by Megatron-LM~\cite{megatronlm-arxiv,megatronlm-sc21}, is considered to be the go-to solution for modern large model training due to strong open-source projects~\cite{megatronlm-github,deepspeed-github,gpt-neox-github} and relatively \revised{low} implementation complexity.
Extending this, Alpa~\cite{alpa-osdi22}, GSPMD~\cite{spmd_socc22}, and nnScaler~\cite{nnscaler-osdi24} provide automatic parallelization for general DNNs.
However, energy consumption is not an optimization metric for any of the major large model training frameworks.

Several works utilize computation idle times within large model training pipelines to insert \revised{additional} useful work~\cite{bamboo-nsdi23,pipefisher-mlsys23,hydro-osdi23}.
In contrast, Perseus's approach is to slow down preceding computations to fill the idle time and reduce power and energy consumption.
\revised{The rationale is twofold.
First, the amount of idle time--especially in the steady state of the pipeline--is typically short.
Inserting extra work may slow down the entire pipeline or require pipeline re-partitioning to make it feasible.
Second, GPUs, especially those in the earlier stage of the pipeline, are likely already constrained by memory capacity, leaving little headroom for extra computation.}
\revised{We note that} even with extra computation inserted, some idle time is likely to remain, leaving room for Perseus to \revised{provide} energy savings.

\paragraph{Deep Learning and Energy Consumption.}
A line of \revised{work measures or estimates} the large amount of energy consumption and carbon emission of \revised{deep learning} workloads~\cite{patterson2021carbon,energy-nlp-policy,quantify-carbon,dodge-faact22,bloomcarbon-arxiv23}.
In terms of optimization, some works determine the GPU's execution speed for a single fixed sequence of GPU computations~\cite{odpp-ccgrid20,datadriven-ccgrid20,gpoeo-tpds21,chase-ccai23,depo-future23,cofris-igsc23}, falling short when the time--energy frontier of a complex large model computation DAG needs \revised{to be} characterized.
Zeus~\cite{zeus-nsdi23} is a recent work that observes the tradeoff between GPU computation time and energy consumption, but still focuses on single-GPU training.
EnvPipe~\cite{envpipe-atc23}, on the other hand, aims to find a point solution that reduces the energy consumption of large model training with minimum slowdown.
However, \revised{EnvPipe's} heuristic assumes that the last pipeline stage is always the bottleneck, leading to suboptimal savings.
\revised{In terms of optimizing a single training job,} Perseus is a superset of both Zeus and EnvPipe, achieved by viewing large model training as a computation DAG and introducing a principled optimization algorithm.

\paragraph{Big Data and Energy Consumption.}
There are existing works that change the frequency of CPUs to improve the energy efficiency of \revised{big data} workloads, including those that only execute computation-intensive phases (e.g., map and reduce) with the maximum frequency~\cite{mapreduce-dvfs-igcc11,hadoop-dvfs-fgcs16}, \revised{those that} choose the lowest frequency that can meet a given deadline~\cite{hadoop-dvfs-js17,express-icac17}, or \revised{those that choose} frequencies that balance the completion time of parallel tasks (yet simply minimizing idle time for smaller tasks within one stage)~\cite{spark-dvfs-js21}.
Perseus provides a principled approach \revised{to} setting the job's end-to-end completion time (\S\ref{sec:overview-optimization}) and task execution speed (\S\ref{sec:design}), and can be directly applied to \revised{big data} workloads by generating a DAG of CPU computations and optimizing the CPU's \revised{frequencies}.

\section{Conclusion}\label{sec:conclusion}

\looseness=-1
We presented Perseus, a software-based energy optimization system for large model training.
Perseus builds on the observation that there are computation imbalances at different levels in large model training that \revised{cause} \emph{intrinsic and extrinsic energy bloat}, and \revised{simultaneously reduces both with a principled graph cut-based algorithm.}
As a result, Perseus advances the state-of-the-art of DNN training energy optimization by delivering energy savings with negligible slowdown, thereby \revised{also} reducing average power draw and enhancing the sustainability of large model training.

\revised{\section*{Acknowledgements}
We would like to thank the SOSP reviewers, our shepherd Eddie Kohler, and SymbioticLab members for their insightful feedback. 
This work is in part supported by NSF grants CNS-2104243, CNS-2106184, and CCF-2327011, NWO VICI grant 639.023.812, grants from VMware and the Mozilla Foundation, and a gift from Salesforce. 
We also thank Chameleon Cloud for providing A100 nodes as well as CloudLab. 
Jae-Won Chung is additionally supported by the Kwanjeong Educational Foundation.}
\label{EndOfPaper}

\bibliographystyle{plain}
\bibliography{ref}

\begin{thebibliography}{10}

\bibitem{deepspeed-github}
{DeepSpeed}.
\newblock \url{https://github.com/microsoft/DeepSpeed}.

\bibitem{us-household}
{{How much electricity does an American home use?}}
\newblock \url{https://www.eia.gov/tools/faqs/faq.php?id=97&t=3}.

\bibitem{megatronlm-github}
{Megatron-LM}.
\newblock \url{https://github.com/NVIDIA/Megatron-LM}.

\bibitem{nvml}
{{NVIDIA Management Library (NVML)}}.
\newblock \url{https://developer.nvidia.com/nvidia-management-library-nvml}.

\bibitem{falcon-arxiv23}
Ebtesam Almazrouei, Hamza Alobeidli, Abdulaziz Alshamsi, Alessandro Cappelli,
  Ruxandra Cojocaru, M{\'e}rouane Debbah, {\'E}tienne Goffinet, Daniel Hesslow,
  Julien Launay, Quentin Malartic, et~al.
\newblock The falcon series of open language models.
\newblock {\em arXiv preprint arXiv:2311.16867}, 2023.

\bibitem{gpt-neox-github}
Alex Andonian, Quentin Anthony, Stella Biderman, Sid Black, Preetham Gali, Leo
  Gao, Eric Hallahan, Josh Levy-Kramer, Connor Leahy, Lucas Nestler, Kip
  Parker, Michael Pieler, Jason Phang, Shivanshu Purohit, Hailey Schoelkopf,
  Dashiell Stander, Tri Songz, Curt Tigges, Benjamin Thérien, Phil Wang, and
  Samuel Weinbach.
\newblock {GPT-NeoX: Large Scale Autoregressive Language Modeling in PyTorch},
  9 2023.

\bibitem{puma-asplos19}
Aayush Ankit, Izzat~El Hajj, Sai~Rahul Chalamalasetti, Geoffrey Ndu, Martin
  Foltin, R.~Stanley Williams, Paolo Faraboschi, Wen-mei~W Hwu, John~Paul
  Strachan, Kaushik Roy, and Dejan~S. Milojicic.
\newblock {PUMA}: A programmable ultra-efficient memristor-based accelerator
  for machine learning inference.
\newblock In {\em ASPLOS}, 2019.

\bibitem{gpt3}
Tom Brown, Benjamin Mann, Nick Ryder, Melanie Subbiah, Jared~D Kaplan, Prafulla
  Dhariwal, Arvind Neelakantan, Pranav Shyam, Girish Sastry, Amanda Askell,
  Sandhini Agarwal, Ariel Herbert-Voss, Gretchen Krueger, Tom Henighan, Rewon
  Child, Aditya Ramesh, Daniel Ziegler, Jeffrey Wu, Clemens Winter, Chris
  Hesse, Mark Chen, Eric Sigler, Mateusz Litwin, Scott Gray, Benjamin Chess,
  Jack Clark, Christopher Berner, Sam McCandlish, Alec Radford, Ilya Sutskever,
  and Dario Amodei.
\newblock Language models are few-shot learners.
\newblock In {\em NeurIPS}, 2020.

\bibitem{openai-keynote-hotchips24}
Trevor Cai.
\newblock Predictable scaling and infrastructure ({HotChips} 2024 keynote
  talk).
\newblock {OpenAI}.

\bibitem{hadoop-dvfs-js17}
Xiaojun Cai, Feng Li, Ping Li, Lei Ju, and Zhiping Jia.
\newblock Sla-aware energy-efficient scheduling scheme for hadoop yarn.
\newblock {\em The Journal of Supercomputing}, 73(8):3526--3546, Aug 2017.

\bibitem{cbre2023}
{CBRE}.
\newblock Global data center trends 2023.
\newblock
  \url{https://www.cbre.com/insights/reports/global-data-center-trends-2023},
  2023.

\bibitem{cbre2024}
{CBRE}.
\newblock Global data center trends 2024.
\newblock
  \url{https://www.cbre.com/insights/reports/global-data-center-trends-2024},
  2024.

\bibitem{checkpointing-arxiv16}
Tianqi Chen, Bing Xu, Chiyuan Zhang, and Carlos Guestrin.
\newblock Training deep nets with sublinear memory cost.
\newblock 2016.

\bibitem{eyeriss-jssc16}
Yu-Hsin Chen, Tushar Krishna, Joel~S. Emer, and Vivienne Sze.
\newblock Eyeriss: An energy-efficient reconfigurable accelerator for deep
  convolutional neural networks.
\newblock {\em IEEE Journal of Solid-State Circuits}, 52(1):127--138, 2017.

\bibitem{envpipe-atc23}
Sangjin Choi, Inhoe Koo, Jeongseob Ahn, Myeongjae Jeon, and Youngjin Kwon.
\newblock {EnvPipe}: Performance-preserving {DNN} training framework for saving
  energy.
\newblock In {\em ATC}, 2023.

\bibitem{cofris-igsc23}
Marcus Chow and Daniel Wong.
\newblock {CoFRIS}: Coordinated frequency and resource scaling for {GPU}
  inference servers.
\newblock In {\em IGSC}, 2024.

\bibitem{crosslayer-energy-eecs24}
Jae-Won Chung, Nishil Talati, and Mosharaf Chowdhury.
\newblock Toward cross-layer energy optimizations in {AI} systems.
\newblock In {\em Energy-Efficient Computing for Science Workshop}, 2024.

\bibitem{mlperf_training_31}
ML~COMMONS.
\newblock {MLPerf} training v3.1 benchmark results.
\newblock https://github.com/mlcommons/training\_results\_v3.1.

\bibitem{bert}
Jacob Devlin, Ming-Wei Chang, Kenton Lee, and Kristina Toutanova.
\newblock {BERT}: Pre-training of deep bidirectional transformers for language
  understanding.
\newblock In {\em Proceedings of the 2019 Conference of the North {A}merican
  Chapter of the Association for Computational Linguistics (NAACL)}, 2019.

\bibitem{dodge-faact22}
Jesse Dodge, Taylor Prewitt, Remi Tachet~des Combes, Erika Odmark, Roy
  Schwartz, Emma Strubell, Alexandra~Sasha Luccioni, Noah~A. Smith, Nicole
  DeCario, and Will Buchanan.
\newblock Measuring the carbon intensity of ai in cloud instances.
\newblock In {\em 2022 ACM Conference on Fairness, Accountability, and
  Transparency}, 2022.

\bibitem{edmondskarp1972}
Jack Edmonds and Richard~M. Karp.
\newblock Theoretical improvements in algorithmic efficiency for network flow
  problems.
\newblock {\em Journal of the ACM}, 19(2):248–264, 1972.

\bibitem{boundedflow}
Jeff Erickson.
\newblock Extensions of maximum flow.
\newblock
  \url{https://courses.engr.illinois.edu/cs498dl1/sp2015/notes/25-maxflowext.pdf}.
\newblock [Online; accessed 05-April-2023].

\bibitem{dapple-ppopp21}
Shiqing Fan, Yi~Rong, Chen Meng, Zongyan Cao, Siyu Wang, Zhen Zheng, Chuan Wu,
  Guoping Long, Jun Yang, Lixue Xia, Lansong Diao, Xiaoyong Liu, and Wei Lin.
\newblock {DAPPLE}: A pipelined data parallel approach for training large
  models.
\newblock In {\em ACM PPoPP}, 2021.

\bibitem{fordfulkerson}
L.~R. Ford and D.~R. Fulkerson.
\newblock {\em Flows in Networks}.
\newblock Princeton University Press, 1962.

\bibitem{recycle-sosp24}
Swapnil Gandhi, Mark Zhao, Athinagoras Skiadopoulos, and Christos Kozyrakis.
\newblock {ReCycle}: Pipeline adaptation for the resilient distributed training
  of large {DNNs}.
\newblock In {\em SOSP}, 2024.

\bibitem{largebatch2}
Noah Golmant, Nikita Vemuri, Zhewei Yao, Vladimir Feinberg, Amir Gholami, Kai
  Rothauge, Michael~W Mahoney, and Joseph Gonzalez.
\newblock On the computational inefficiency of large batch sizes for stochastic
  gradient descent.
\newblock {\em arXiv preprint arXiv:1811.12941}, 2018.

\bibitem{clockwork-osdi20}
Arpan Gujarati, Reza Karimi, Safya Alzayat, Wei Hao, Antoine Kaufmann, Ymir
  Vigfusson, and Jonathan Mace.
\newblock Serving {DNNs} like clockwork: Performance predictability from the
  bottom up.
\newblock In {\em OSDI}, 2020.

\bibitem{largescalefailures-sc17}
Saurabh Gupta, Tirthak Patel, Christian Engelmann, and Devesh Tiwari.
\newblock Failures in large scale systems: Long-term measurement, analysis, and
  implications.
\newblock In {\em SC}, 2017.

\bibitem{hamilton2024constraint}
James Hamilton.
\newblock Constraint-driven innovation ({CIDR} 2024 keynote talk).
\newblock
  \url{https://mvdirona.com/jrh/talksandpapers/JamesHamiltonCIDR2024.pdf}.

\bibitem{autopipe-icml21}
Chaoyang He, Shen Li, Mahdi Soltanolkotabi, and Salman Avestimehr.
\newblock {PipeTransformer}: Automated elastic pipelining for distributed
  training of large-scale models.
\newblock In {\em ICML}, 2021.

\bibitem{pd-caie16}
Dorit~S. Hochbaum.
\newblock A polynomial time repeated cuts algorithm for the time cost tradeoff
  problem: The linear and convex crashing cost deadline problem.
\newblock {\em Computers \& Industrial Engineering}, 95:64--71, 2016.

\bibitem{chinchilla-neurips22}
Jordan Hoffmann, Sebastian Borgeaud, Arthur Mensch, Elena Buchatskaya, Trevor
  Cai, Eliza Rutherford, Diego de~Las~Casas, Lisa~Anne Hendricks, Johannes
  Welbl, Aidan Clark, Thomas Hennigan, Eric Noland, Katherine Millican, George
  van~den Driessche, Bogdan Damoc, Aurelia Guy, Simon Osindero, Kar\'{e}n
  Simonyan, Erich Elsen, Oriol Vinyals, Jack Rae, and Laurent Sifre.
\newblock An empirical analysis of compute-optimal large language model
  training.
\newblock In {\em NeurIPS}, 2022.

\bibitem{hydro-osdi23}
Qinghao Hu, Zhisheng Ye, Meng Zhang, Qiaoling Chen, Peng Sun, Yonggang Wen, and
  Tianwei Zhang.
\newblock Hydro: {Surrogate-Based} hyperparameter tuning service in
  datacenters.
\newblock In {\em OSDI}, 2023.

\bibitem{gpipe-neurips19}
Yanping Huang, Youlong Cheng, Ankur Bapna, Orhan Firat, Mia~Xu Chen, Dehao
  Chen, HyoukJoong Lee, Jiquan Ngiam, Quoc~V. Le, Yonghui Wu, and Zhifeng Chen.
\newblock {GPipe}: Efficient training of giant neural networks using pipeline
  parallelism.
\newblock In {\em NeurIPS}, 2019.

\bibitem{hadoop-dvfs-fgcs16}
Shadi Ibrahim, Tien-Dat Phan, Alexandra Carpen-Amarie, Houssem-Eddine Chihoub,
  Diana Moise, and Gabriel Antoniu.
\newblock Governing energy consumption in hadoop through cpu frequency scaling:
  An analysis.
\newblock {\em Future Generation Computer Systems}, 54:219--232, 2016.

\bibitem{datadriven-ccgrid20}
Shashikant Ilager, Rajeev Muralidhar, Kotagiri Rammohanrao, and Rajkumar Buyya.
\newblock A data-driven frequency scaling approach for deadline-aware energy
  efficient scheduling on graphics processing units ({GPUs}).
\newblock In {\em CCGRID}, 2020.

\bibitem{oobleck-sosp23}
Insu Jang, Zhenning Yang, Zhen Zhang, Xin Jin, and Mosharaf Chowdhury.
\newblock Oobleck: Resilient distributed training of large models using
  pipeline templates.
\newblock In {\em SOSP}, 2023.

\bibitem{philly-atc19}
Myeongjae Jeon, Shivaram Venkataraman, Amar Phanishayee, Junjie Qian, Wencong
  Xiao, and Fan Yang.
\newblock Analysis of {Large-Scale} {Multi-Tenant} {GPU} clusters for {DNN}
  training workloads.
\newblock In {\em ATC}, 2019.

\bibitem{megascale-nsdi24}
Ziheng Jiang, Haibin Lin, Yinmin Zhong, Qi~Huang, Yangrui Chen, Zhi Zhang,
  Yanghua Peng, Xiang Li, Cong Xie, Shibiao Nong, Yulu Jia, Sun He, Hongmin
  Chen, Zhihao Bai, Qi~Hou, Shipeng Yan, Ding Zhou, Yiyao Sheng, Zhuo Jiang,
  Haohan Xu, Haoran Wei, Zhang Zhang, Pengfei Nie, Leqi Zou, Sida Zhao, Liang
  Xiang, Zherui Liu, Zhe Li, Xiaoying Jia, Jianxi Ye, Xin Jin, and Xin Liu.
\newblock {MegaScale}: Scaling large language model training to more than
  10,000 {GPUs}.
\newblock In {\em NSDI}, 2024.

\bibitem{scalinglaws-arxiv20}
Jared Kaplan, Sam McCandlish, Tom Henighan, Tom~B. Brown, Benjamin Chess, Rewon
  Child, Scott Gray, Alec Radford, Jeffrey Wu, and Dario Amodei.
\newblock Scaling laws for neural language models.
\newblock {\em arXiv preprint arXiv:2001.08361}, 2020.

\bibitem{chameleoncloud}
Kate Keahey, Jason Anderson, Zhuo Zhen, Pierre Riteau, Paul Ruth, Dan
  Stanzione, Mert Cevik, Jacob Colleran, Haryadi~S Gunawi, Cody Hammock, et~al.
\newblock Lessons learned from the chameleon testbed.
\newblock In {\em ATC}, 2020.

\bibitem{largebatch1}
Nitish~Shirish Keskar, Dheevatsa Mudigere, Jorge Nocedal, Mikhail Smelyanskiy,
  and Ping Tak~Peter Tang.
\newblock On large-batch training for deep learning: Generalization gap and
  sharp minima.
\newblock In {\em ICLR}, 2017.

\bibitem{depo-future23}
Adam Krzywaniak, Paweł Czarnul, and Jerzy Proficz.
\newblock Dynamic {GPU} power capping with online performance tracing for
  energy efficient {GPU} computing using {DEPO} tool.
\newblock {\em Future Generation Computer Systems}, 145:396--414, 2023.

\bibitem{quantify-carbon}
Alexandre Lacoste, Alexandra Luccioni, Victor Schmidt, and Thomas Dandres.
\newblock Quantifying the carbon emissions of machine learning.
\newblock {\em arXiv preprint arXiv:1910.09700}, 2019.

\bibitem{merak-tpds23}
Zhiquan Lai, Shengwei Li, Xudong Tang, Keshi Ge, Weijie Liu, Yabo Duan, Linbo
  Qiao, and Dongsheng Li.
\newblock Merak: An efficient distributed {DNN} training framework with
  automated 3d parallelism for giant foundation models.
\newblock {\em IEEE Transactions on Parallel and Distributed Systems},
  34(5):1466--1478, 2023.

\bibitem{spark-dvfs-js21}
Hongjian Li, Yaojun Wei, Yu~Xiong, Enjie Ma, and Wenhong Tian.
\newblock A frequency-aware and energy-saving strategy based on dvfs for spark.
\newblock {\em The Journal of Supercomputing}, 77(10):11575--11596, Oct 2021.

\bibitem{thunderbolt-osdi20}
Shaohong Li, Xi~Wang, Xiao Zhang, Vasileios Kontorinis, Sreekumar Kodakara,
  David Lo, and Parthasarathy Ranganathan.
\newblock Thunderbolt: {Throughput-Optimized}, {Quality-of-Service-Aware} power
  capping at scale.
\newblock In {\em OSDI}, 2020.

\bibitem{nnscaler-osdi24}
Zhiqi Lin, Youshan Miao, Quanlu Zhang, Fan Yang, Yi~Zhu, Cheng Li, Saeed
  Maleki, Xu~Cao, Ning Shang, Yilei Yang, Weijiang Xu, Mao Yang, Lintao Zhang,
  and Lidong Zhou.
\newblock {nnScaler}: {Constraint-Guided} parallelization plan generation for
  deep learning training.
\newblock In {\em OSDI}, 2024.

\bibitem{bloomcarbon-arxiv23}
Alexandra~Sasha Luccioni, Sylvain Viguier, and Anne-Laure Ligozat.
\newblock Estimating the carbon footprint of {BLOOM}, a 176b parameter language
  model.
\newblock 2022.

\bibitem{blackwell-news}
Tobias Mann.
\newblock {NVIDIA} turns up the {AI} heat with 1,200w blackwell gpus.
\newblock https://www.theregister.com/2024/03/18/nvidia\_turns\_up\_the\_ai.

\bibitem{express-icac17}
Stathis Maroulis, Nikos Zacheilas, and Vana Kalogeraki.
\newblock {ExpREsS}: Energy efficient scheduling of mixed stream and batch
  processing workloads.
\newblock In {\em ICAC}, 2017.

\bibitem{aienergy-joule24}
Eric Masanet, Nuoa Lei, and Jonathan Koomey.
\newblock To better understand ai’s growing energy use, analysts need a data
  revolution.
\newblock {\em Joule}, 2024.

\bibitem{mckinsey2023}
{McKinsey \& Company}.
\newblock Investing in the rising data center economy.
\newblock
  \url{https://www.mckinsey.com/industries/technology-media-and-telecommunications/our-insights/investing-in-the-rising-data-center-economy#/},
  2023.

\bibitem{datastalls-vldb21}
Jayashree Mohan, Amar Phanishayee, Ashish Raniwala, and Vijay Chidambaram.
\newblock Analyzing and mitigating data stalls in dnn training.
\newblock 14(5):771–784, jan 2021.

\bibitem{pipedream-sosp19}
Deepak Narayanan, Aaron Harlap, Amar Phanishayee, Vivek Seshadri, Nikhil~R
  Devanur, Gregory~R Ganger, Phillip~B Gibbons, and Matei Zaharia.
\newblock {PipeDream}: generalized pipeline parallelism for {DNN} training.
\newblock In {\em SOSP}, 2019.

\bibitem{pipedreamflush-icml21}
Deepak Narayanan, Amar Phanishayee, Kaiyu Shi, Xie Chen, and Matei Zaharia.
\newblock Memory-efficient pipeline-parallel {DNN} training.
\newblock In {\em ICML}, 2021.

\bibitem{megatronlm-sc21}
Deepak Narayanan, Mohammad Shoeybi, Jared Casper, Patrick LeGresley, Mostofa
  Patwary, Vijay Korthikanti, Dmitri Vainbrand, Prethvi Kashinkunti, Julie
  Bernauer, Bryan Catanzaro, Amar Phanishayee, and Matei Zaharia.
\newblock Efficient large-scale language model training on {GPU} clusters using
  {Megatron-LM}.
\newblock In {\em SC}, 2021.

\bibitem{H100whitepaper}
NVIDIA.
\newblock Nvidia {H100} tensor core {GPU} architecture overview.
\newblock
  https://resources.nvidia.com/en-us-tensor-core/gtc22-whitepaper-hopper.

\bibitem{pipefisher-mlsys23}
Kazuki Osawa, Shigang Li, and Torsten Hoefler.
\newblock {PipeFisher}: Efficient training of large language models using
  pipelining and fisher information matrices.
\newblock In {\em MLSys}, 2023.

\bibitem{pytorch}
Adam Paszke, Sam Gross, Francisco Massa, Adam Lerer, James Bradbury, Gregory
  Chanan, Trevor Killeen, Zeming Lin, Natalia Gimelshein, Luca Antiga, et~al.
\newblock Pytorch: An imperative style, high-performance deep learning library.
\newblock {\em NeurIPS}, 2019.

\bibitem{polca-asplos24}
Pratyush Patel, Esha Choukse, Chaojie Zhang, \'{I}\~{n}igo Goiri, Brijesh
  Warrier, Nithish Mahalingam, and Ricardo Bianchini.
\newblock Characterizing power management opportunities for llms in the cloud.
\newblock 2024.

\bibitem{gpupower-cal23}
Pratyush Patel, Zibo Gong, Syeda Rizvi, Esha Choukse, Pulkit Misra, Thomas
  Anderson, and Akshitha Sriraman.
\newblock Towards improved power management in cloud gpus.
\newblock {\em IEEE Computer Architecture Letters}, 22(2):141--144, 2023.

\bibitem{patterson2021carbon}
David Patterson, Joseph Gonzalez, Quoc Le, Chen Liang, Lluis-Miquel Munguia,
  Daniel Rothchild, David So, Maud Texier, and Jeff Dean.
\newblock Carbon emissions and large neural network training.
\newblock {\em arXiv preprint arXiv:2104.10350}, 2021.

\bibitem{pd-mnsc77}
Steve Phillips and Mohamed~I. Dessouky.
\newblock Solving the project time/cost tradeoff problem using the minimal cut
  concept.
\newblock {\em Management Science}, 24(4):393--400, 1977.

\bibitem{dvfs-boosting-asplos24}
Leonardo Piga, Iyswarya Narayanan, Aditya Sundarrajan, Matt Skach, Qingyuan
  Deng, Biswadip Maity, Manoj Chakkaravarthy, Alison Huang, Abhishek Dhanotia,
  and Parth Malani.
\newblock Expanding datacenter capacity with dvfs boosting: A safe and scalable
  deployment experience.
\newblock In {\em ASPLOS}, 2024.

\bibitem{t5-jmlr20}
Colin Raffel, Noam Shazeer, Adam Roberts, Katherine Lee, Sharan Narang, Michael
  Matena, Yanqi Zhou, Wei Li, and Peter~J. Liu.
\newblock Exploring the limits of transfer learning with a unified text-to-text
  transformer.
\newblock {\em Journal of Machine Learning Research}, 21(140):1--67, 2020.

\bibitem{mvpp-asplos20}
Varun Sakalkar, Vasileios Kontorinis, David Landhuis, Shaohong Li, Darren
  De~Ronde, Thomas Blooming, Anand Ramesh, James Kennedy, Christopher Malone,
  Jimmy Clidaras, and Parthasarathy Ranganathan.
\newblock Data center power oversubscription with a medium voltage power plane
  and {Priority-Aware} capping.
\newblock In {\em ASPLOS}, 2020.

\bibitem{greenai-cacm20}
Roy Schwartz, Jesse Dodge, Noah~A. Smith, and Oren Etzioni.
\newblock Green {AI}.
\newblock {\em Commun. ACM}, 63(12):54–63, 2020.

\bibitem{megatronlm-arxiv}
Mohammad Shoeybi, Mostofa Patwary, Raul Puri, Patrick LeGresley, Jared Casper,
  and Bryan Catanzaro.
\newblock {Megatron-LM}: Training multi-billion parameter language models using
  model parallelism.
\newblock {\em arXiv preprint arXiv:1909.08053}, 2019.

\bibitem{skutella1998}
Martin Skutella.
\newblock Approximation algorithms for the discrete time-cost tradeoff problem.
\newblock {\em Mathematics of Operations Research}, 23(4):909--929, 1998.

\bibitem{skutella-thesis}
Martin Skutella.
\newblock {\em Approximation and randomization in scheduling}.
\newblock PhD thesis, 1998.

\bibitem{megatron-turing-arxiv22}
Shaden Smith, Mostofa Patwary, Brandon Norick, Patrick LeGresley, Samyam
  Rajbhandari, Jared Casper, Zhun Liu, Shrimai Prabhumoye, George Zerveas,
  Vijay Korthikanti, Elton Zhang, Rewon Child, Reza~Yazdani Aminabadi, Julie
  Bernauer, Xia Song, Mohammad Shoeybi, Yuxiong He, Michael Houston, Saurabh
  Tiwary, and Bryan Catanzaro.
\newblock Using deepspeed and megatron to train {Megatron-Turing} {NLG} 530b, a
  {Large-Scale} generative language model.
\newblock 2022.

\bibitem{energy-nlp-policy}
Emma Strubell, Ananya Ganesh, and Andrew McCallum.
\newblock Energy and policy considerations for deep learning in {NLP}.
\newblock {\em Proceedings of the 57th Annual Meeting of the Association for
  Computational Linguistics}, 2019.

\bibitem{svensson2012hardness}
Ola Svensson.
\newblock Hardness of vertex deletion and project scheduling.
\newblock In {\em International Workshop on Approximation Algorithms for
  Combinatorial Optimization}, pages 301--312. Springer, 2012.

\bibitem{vivienne17efficient}
Vivienne Sze, Yu-Hsin Chen, Tien-Ju Yang, and Joel~S. Emer.
\newblock Efficient processing of deep neural networks: A tutorial and survey.
\newblock {\em Proceedings of the IEEE}, 105(12):2295--2329, 2017.

\bibitem{bamboo-nsdi23}
John Thorpe, Pengzhan Zhao, Jonathan Eyolfson, Yifan Qiao, Zhihao Jia, Minjia
  Zhang, Ravi Netravali, and Guoqing~Harry Xu.
\newblock Bamboo: Making preemptible instances resilient for affordable
  training of large {DNNs}.
\newblock In {\em NSDI}, 2023.

\bibitem{transformer-neurips17}
Ashish Vaswani, Noam Shazeer, Niki Parmar, Jakob Uszkoreit, Llion Jones,
  Aidan~N. Gomez, \L{}ukasz Kaiser, and Illia Polosukhin.
\newblock Attention is all you need.
\newblock In {\em NeurIPS}, 2017.

\bibitem{gpoeo-tpds21}
Farui Wang, Weizhe Zhang, Shichao Lai, Meng Hao, and Zheng Wang.
\newblock Dynamic {GPU} energy optimization for machine learning training
  workloads.
\newblock {\em IEEE Transactions on Parallel and Distributed Systems}, 2021.

\bibitem{mlaas-nsdi22}
Qizhen Weng, Wencong Xiao, Yinghao Yu, Wei Wang, Cheng Wang, Jian He, Yong Li,
  Liping Zhang, Wei Lin, and Yu~Ding.
\newblock {MLaaS} in the wild: Workload analysis and scheduling in large-scale
  heterogeneous {GPU} clusters.
\newblock In {\em NSDI}, 2022.

\bibitem{mapreduce-dvfs-igcc11}
Thomas Wirtz and Rong Ge.
\newblock Improving {MapReduce} energy efficiency for computation intensive
  workloads.
\newblock In {\em International Green Computing Conference and Workshops},
  2011.

\bibitem{huggingface}
Thomas Wolf, Lysandre Debut, Victor Sanh, Julien Chaumond, Clement Delangue,
  Anthony Moi, Pierric Cistac, Tim Rault, Remi Louf, Morgan Funtowicz, Joe
  Davison, Sam Shleifer, Patrick von Platen, Clara Ma, Yacine Jernite, Julien
  Plu, Canwen Xu, Teven Le~Scao, Sylvain Gugger, Mariama Drame, Quentin Lhoest,
  and Alexander Rush.
\newblock Transformers: State-of-the-art natural language processing.
\newblock In {\em EMNLP}, 2020.

\bibitem{bloom-arxiv22}
BigScience Workshop.
\newblock {BLOOM}: A 176b-parameter open-access multilingual language model.
\newblock 2023.

\bibitem{antdt-icde24}
Y.~Xiao, L.~Ju, Z.~Zhou, S.~Li, Z.~Huan, D.~Zhang, R.~Jiang, L.~Wang, X.~Zhang,
  L.~Liang, and J.~Zhou.
\newblock Antdt: A self-adaptive distributed training framework for leader and
  straggler nodes.
\newblock In {\em ICDE}, 2024.

\bibitem{kws-jssc24}
Heejin Yang, Ji{-}Hwan Seol, Rohit Rothe, Zichen Fan, Qirui Zhang, Hun{-}Seok
  Kim, David~T. Blaauw, and Dennis Sylvester.
\newblock A 1.5-{$\mu$W} fully-integrated keyword spotting soc in 28-nm {CMOS}
  with skip-rnn and fast-settling analog frontend for adaptive frame skipping.
\newblock {\em {IEEE} J. Solid State Circuits}, 59(1):29--39, 2024.

\bibitem{chase-ccai23}
Zhenning Yang, Luoxi Meng, Jae-Won Chung, and Mosharaf Chowdhury.
\newblock Chasing low-carbon electricity for practical and sustainable dnn
  training.
\newblock 2023.

\bibitem{zeus-nsdi23}
Jie You, Jae-Won Chung, and Mosharaf Chowdhury.
\newblock Zeus: Understanding and optimizing {GPU} energy consumption of {DNN}
  training.
\newblock In {\em USENIX NSDI}, 2023.

\bibitem{wideresnet-bmvc16}
Sergey Zagoruyko and Nikos Komodakis.
\newblock Wide residual networks.
\newblock In {\em Proceedings of the British Machine Vision Conference (BMVC)},
  2016.

\bibitem{spmd_socc22}
Shiwei Zhang, Lansong Diao, Chuan Wu, Siyu Wang, and Wei Lin.
\newblock Accelerating large-scale distributed neural network training with
  spmd parallelism.
\newblock In {\em SoCC}, 2022.

\bibitem{recsys-dsi-isca22}
Mark Zhao, Niket Agarwal, Aarti Basant, Bu\u{g}ra Gedik, Satadru Pan, Mustafa
  Ozdal, Rakesh Komuravelli, Jerry Pan, Tianshu Bao, Haowei Lu, Sundaram
  Narayanan, Jack Langman, Kevin Wilfong, Harsha Rastogi, Carole-Jean Wu,
  Christos Kozyrakis, and Parik Pol.
\newblock Understanding data storage and ingestion for large-scale deep
  recommendation model training: Industrial product.
\newblock In {\em ISCA}, 2022.

\bibitem{alpa-osdi22}
Lianmin Zheng, Zhuohan Li, Hao Zhang, Yonghao Zhuang, Zhifeng Chen, Yanping
  Huang, Yida Wang, Yuanzhong Xu, Danyang Zhuo, Eric~P. Xing, Joseph~E.
  Gonzalez, and Ion Stoica.
\newblock Alpa: Automating inter- and {Intra-Operator} parallelism for
  distributed deep learning.
\newblock In {\em USENIX OSDI}, 2022.

\bibitem{odpp-ccgrid20}
Pengfei Zou, Ang Li, Kevin Barker, and Rong Ge.
\newblock Indicator-directed dynamic power management for iterative workloads
  on {GPU}-accelerated systems.
\newblock In {\em CCGRID}, 2020.

\bibitem{tpuv4-nsdi24}
Yazhou Zu, Alireza Ghaffarkhah, Hoang-Vu Dang, Brian Towles, Steven Hand,
  Safeen Huda, Adekunle Bello, Alexander Kolbasov, Arash Rezaei, Dayou Du,
  Steve Lacy, Hang Wang, Aaron Wisner, Chris Lewis, and Henri Bahini.
\newblock Resiliency at scale: Managing {Google's} {TPUv4} machine learning
  supercomputer.
\newblock In {\em NSDI}, 2024.

\bibitem{A100powerthermal-energies21}
Matej Špeťko, Ondřej Vysocký, Branislav Jansík, and Lubomír Říha.
\newblock {DGX-A100} face to face {DGX}-2 -- performance, power and thermal
  behavior evaluation.
\newblock {\em Energies}, 14(2), 2021.

\end{thebibliography}

\clearpage

\appendix

\noindent\fbox{%
  \parbox{0.95\linewidth}{%
    \centering
    \textit{The Appendix has not been peer-reviewed\\and is provided for supplementary information.}%
  }%
}

\section{Visualizations for Intrinsic Energy Bloat}\label{sec:appendix-intrinsic-bloat}
  
\begin{figure*}[!t]
    \centering
    \subfloat[BERT 1.3B]{
      \includegraphics[width=0.44\linewidth]{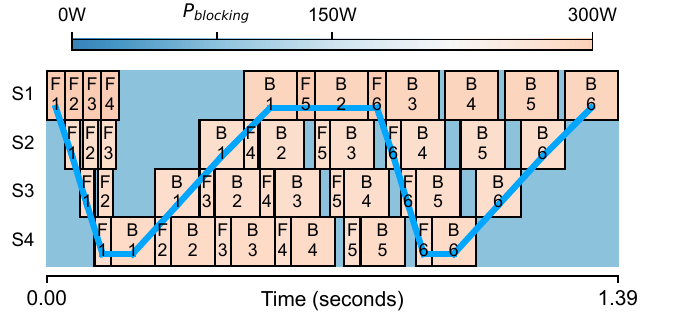}
      \includegraphics[width=0.44\linewidth]{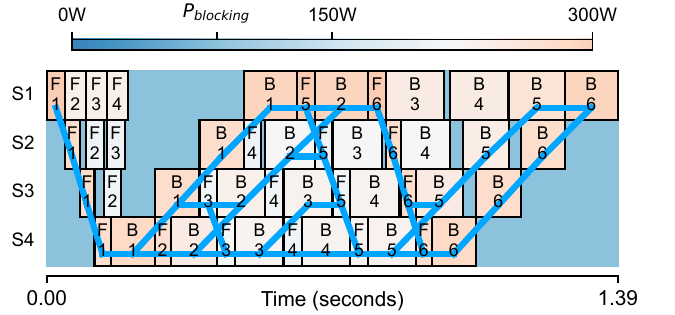}
    }

    \subfloat[T5 3B]{
      \includegraphics[width=0.44\linewidth]{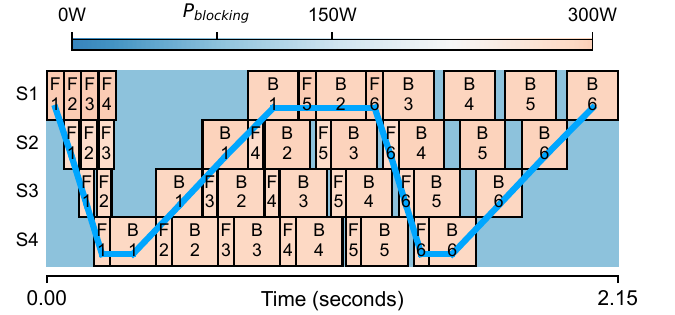}
      \includegraphics[width=0.44\linewidth]{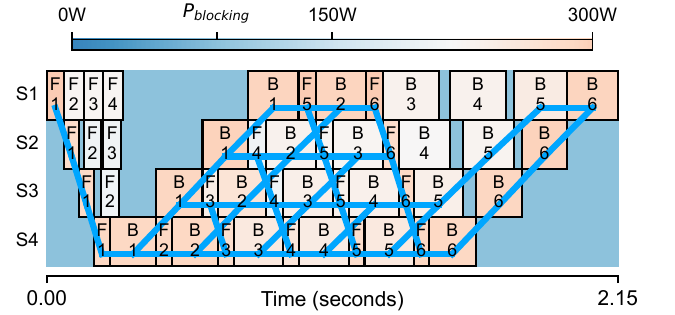}
    }
    
    \subfloat[Bloom 3B]{
      \includegraphics[width=0.44\linewidth]{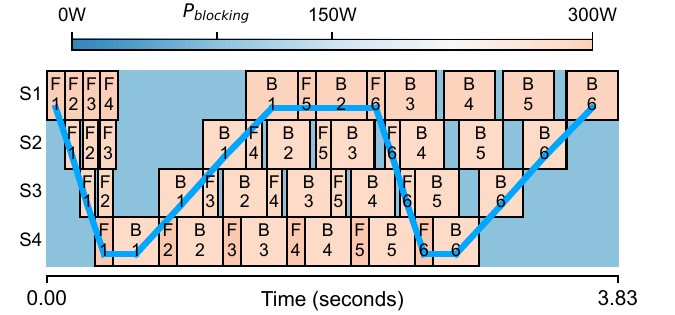}
       
      \includegraphics[width=0.44\linewidth]{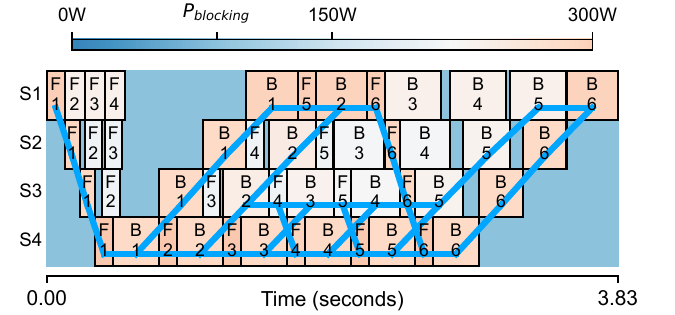}
    }

    \subfloat[Wide-ResNet101 1.5B]{
      \includegraphics[width=0.44\linewidth]{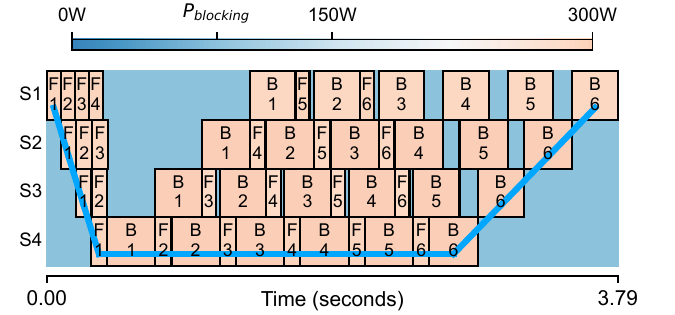}
      \includegraphics[width=0.44\linewidth]{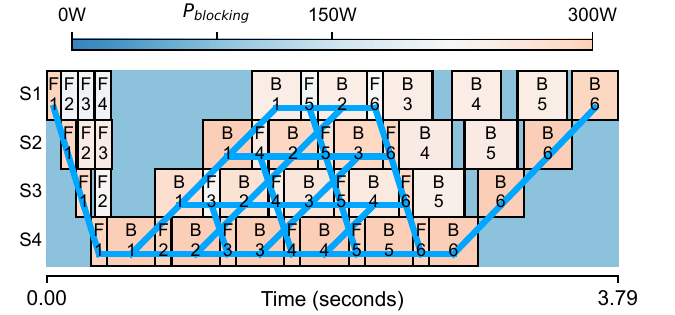}
    }

    \caption{Visualization of Perseus's $T_{\min}$ solution for four stage pipeline workloads on NVIDIA A100 PCIe GPUs. For each workload, the left is running every computation at maximum frequency, and the right is Perseus's energy schedule that reduces only intrinsic energy bloat without inflating iteration time. Note that these are \emph{not} real workloads we run in Section~\ref{sec:eval}; real workloads have far more microbatches.}\label{fig:appendix-intrinsic-bloat}
\end{figure*}

Figure~\ref{fig:appendix-intrinsic-bloat} shows the timeline of running one training iteration of BERT 1.3B, T5 3B, Bloom 3B, and Wide-ResNet101 1.5B for maximum frequency and Perseus-optimized energy schedule, respectively. 
For visualization purposes, we set the number of microbatches to 6.
Real evaluation workloads have more microbatches.
Energy schedule found by Perseus successfully tunes down frequency for all models without slowing down the iteration time, tightly packing computations over time and reducing intrinsic energy bloat.

\section{Workload Details}\label{sec:appendix-workloads}

\subsection{Minimum Imbalance Pipeline Partitioning}

\begin{table*}[!t]
  \footnotesize
	\centering
  \begin{subtable}[c]{\textwidth}
  \centering
  \begin{tabular}{lrrrll}
    \toprule
    \multirow{2}{*}{\textbf{Model}} & \multirow{2}{*}{\textbf{Size}} & \multicolumn{2}{c}{\textbf{Imbalance Ratio}} & \multicolumn{2}{c}{\textbf{Minimum Imbalance Ratio Partition}} \\ 
                                    & & 4 stages & 8 stages & 4 stages & 8 stages \\
    \midrule
    \multirow{4}{*}{GPT-3~\cite{gpt3}}          &   1B  & 1.17  & 1.33  & [0, 6, 12, 19, 25]  & [0, 4, 7, 10, 13, 16, 19, 22, 25] \\
                                                &   3B  & 1.13  & 1.25  & [0, 8, 16, 25, 33]  & [0, 5, 9, 13, 17, 21, 25, 29, 33] \\
                                                &   7B  & 1.11  & 1.23  & [0, 8, 16, 24, 33]  & [0, 4, 8, 12, 16, 20, 24, 28, 33]  \\
                                                &  13B  & 1.08  & 1.17  & [0, 10, 20, 30, 41]  & [0, 5, 10, 15, 20, 25, 30, 35, 41]  \\
                                                & 175B  & 1.02  & 1.03  & [0, 24, 48, 72, 97]  & [0, 12, 24, 36, 48, 60, 72, 84, 97]  \\
    \midrule
    \multirow{3}{*}{Bloom~\cite{bloom-arxiv22}} &   3B  & 1.13  & 1.25  & [0, 9, 17, 25, 31]  & [0, 6, 11, 16, 22, 28, 31]  \\
                                                &   7B  & 1.13  & 1.25  & [0, 9, 17, 25, 31]  & [0, 6, 11, 16, 22, 28, 31]  \\
                                                & 176B  & 1.05  & 1.10  & [0, 18, 36, 54, 71]  & [0, 9, 18, 27, 36, 45, 54, 63, 71]  \\
    \midrule
    \multirow{3}{*}{BERT~\cite{bert}}           & 0.1B  & 1.33  & 2.00  & [0, 4, 7, 10, 13]  & [0, 2, 3, 4, 6, 8, 10, 12, 13]  \\
                                                & 0.3B  & 1.17  & 1.33  & [0, 7, 13, 19, 25]  & [0, 3, 6, 9, 12, 15, 18, 22, 25]  \\
                                                & 1.3B  & 1.17  & 1.33  & [0, 7, 13, 19, 25]  & [0, 4, 7, 10, 13, 16, 19, 22, 25]  \\
    \midrule
    \multirow{3}{*}{T5~\cite{t5-jmlr20}}        & 0.2B  & 1.19  & 1.50  & [0, 9, 15, 20, 25]  & [0, 5, 9, 13, 15, 17, 19, 22, 25]  \\
                                                & 0.7B  & 1.05  & 1.11  & [0, 16, 29, 39, 49]  & [0, 8, 16, 24, 29, 34, 39, 44, 49]  \\
                                                & 2.9B  & 1.06  & 1.16  & [0, 15, 28, 38, 49]  & [0, 7, 15, 23, 28, 33, 38, 43, 49]  \\
    \midrule
    Wide-ResNet50~\cite{wideresnet-bmvc16}      & 0.8B  & 1.23  & 1.46  & [0, 5, 9, 14, 18]  & [0, 3, 5, 7, 9, 11, 13, 15, 18]  \\
    Wide-ResNet101~\cite{wideresnet-bmvc16}     & 1.5B  & 1.09  & 1.25  & [0, 8, 17, 26, 35]  & [0, 4, 8, 12, 16, 21, 26, 31, 35]  \\
    \bottomrule
  \end{tabular}
  \subcaption{NVIDIA A100 PCIe GPUs.}
  \end{subtable}

  \begin{subtable}[c]{\textwidth}
  \centering
  \begin{tabular}{lrrrll}
    \toprule
    \multirow{2}{*}{\textbf{Model}} & \multirow{2}{*}{\textbf{Size}} & \multicolumn{2}{c}{\textbf{Imbalance Ratio}} & \multicolumn{2}{c}{\textbf{Minimum Imbalance Ratio Partition}} \\ 
                                    & & 4 stages & 8 stages & 4 stages & 8 stages \\
    \midrule
    \multirow{4}{*}{GPT-3~\cite{gpt3}}          &   1B  & 1.15  & 1.31  & [0, 6, 12, 18, 25] & [0, 3, 6, 9, 12, 15, 18, 21, 25]  \\
                                                &   3B  & 1.11  & 1.21  & [0, 8, 16, 24, 33]  & [0, 4, 8, 12, 16, 20, 24, 28, 33]  \\
                                                &   7B  & 1.08  & 1.17  & [0, 8, 16, 24, 33]  & [0, 4, 8, 12, 16, 20, 24, 28, 33]  \\
                                                &  13B  & 1.07  & 1.14  & [0, 10, 20, 30, 41]  & [0, 5, 10, 15, 20, 25, 30, 35, 41]  \\
                                                & 175B  & 1.01  & 1.02  & [0, 24, 48, 72, 97]  & [0, 12, 24, 36, 48, 60, 72, 84, 97]  \\
    \midrule
    \multirow{3}{*}{Bloom~\cite{bloom-arxiv22}} &   3B  & 1.13  & 1.25  & [0, 9, 17, 25, 31]  & [0, 5, 9, 13, 17, 21, 25, 29, 31]  \\
                                                &   7B  & 1.13  & 1.25  & [0, 9, 17, 25, 31]  & [0, 5, 9, 13, 17, 21, 25, 29, 31]  \\
                                                & 176B  & 1.03  & 1.06  & [0, 18, 36, 54, 71]  & [0, 9, 18, 27, 36, 45, 54, 63, 71]  \\
    \midrule
    \multirow{3}{*}{BERT~\cite{bert}}           & 0.1B  & 1.33  & 2.00  & [0, 4, 7, 10, 13]  & [0, 1, 2, 4, 6, 8, 10, 12, 13]  \\
                                                & 0.3B  & 1.17  & 1.33  & [0, 7, 13, 19, 25]  & [0, 4, 7, 10, 13, 16, 19, 22, 25]  \\
                                                & 1.3B  & 1.17  & 1.33  & [0, 7, 13, 19, 25]  & [0, 3, 6, 9, 12, 15, 18, 22, 25]  \\
    \midrule
    \multirow{3}{*}{T5~\cite{t5-jmlr20}}        & 0.2B  & 1.20  & 1.50  & [0, 9, 15, 20, 25]  & [0, 5, 9, 13, 15, 17, 19, 22, 25]  \\
                                                & 0.7B  & 1.06  & 1.12  & [0, 16, 29, 39, 49]  & [0, 8, 16, 24, 29, 34, 39, 44, 49]  \\
                                                & 2.9B  & 1.07  & 1.17  & [0, 15, 28, 38, 49]  & [0, 8, 15, 23, 28, 33, 38, 43, 49]  \\
    \midrule
    Wide-ResNet50~\cite{wideresnet-bmvc16}      & 0.8B  & 1.13  & 1.72  & [0, 5, 9, 14, 18]  & [0, 3, 5, 7, 9, 11, 13, 15, 18]  \\
    Wide-ResNet101~\cite{wideresnet-bmvc16}     & 1.5B  & 1.08  & 1.25  & [0, 8, 17, 26, 35]  & [0, 4, 8, 12, 16, 21, 26, 31, 35]  \\
    \bottomrule
  \end{tabular}
  \subcaption{NVIDIA A40 GPUs.}
  \end{subtable}

  \caption{Imbalance ratio between the longest and the shortest stages for various models. 1.00 would mean perfect balance. Partitions for $N$ stages is a list of $N + 1$ numbers, where the numbers represent layer indices. For instance, [0, 6, 12, 19, 25] for GPT-3 1.3B means there are 6, 6, 7, and 5 Transformer layers in each stage, and the final stage also has the language model head.}\label{tab:appendix-partition}
\end{table*}

We partition layers of a model into $N$ stages such that the imbalance ratio, defined as the ratio of the longest stage forward latency to the shortest, is minimized.
We only consider forward computation time as backward computations are typically proportional to forward computation latency.
For Transformer-based models, we define layer as one Transformer layer.
For Wide-ResNet, we define layer as one Bottleneck layer, which is three convolution layers wrapped with a skip connection.
Due to P2P communication overhead and implementation challenges, many planners and frameworks do not support partitioning in the middle of skip connections.
We call this \emph{minimum imbalance pipeline partitioning}, and throughout the paper, every workload we use is partitioned as such.

Table~\ref{tab:appendix-partition} shows the computation time ratios of the heaviest stage to the lightest stage for 4 and 8 pipeline stages.
More pipeline stages generally increases imbalance due to the coarse-grained nature of tensor operations.
That is, the relative size of each layer becomes smaller and smaller compared to the total amount of computation allocated to each stage, and imbalance increases.

\paragraph{GPT-3, Bloom, and BERT.}
Arguably, these are one of the most homogeneous large models because they are a stack of identical Transformer~\cite{transformer-neurips17} encoder or decoder layers.
However, the very last layer is the language modeling head, which maps the final features to probabilities over the entire vocabulary.
The vocabulary size of GPT-3 is 50k, Bloom 251k, and BERT 31k, which results in a very large linear layer for the last stage.
This leads to varying amounts of imbalance and different minimum imbalance partitioning results for each model.

\paragraph{T5.}
This is also based on Transformer layers, but the first half of the layers are encoders, while the later half are decoders (which corresponds to the original Transformer~\cite{transformer-neurips17} model's architecture).
However, the decoder layers as an extra cross attention layer in the middle, making it computationally heavier.
Finally, T5 also ends with a language model head with 32k vocabulary size.
However, minimum imbalance partitioning still balances T5 to a reasonable degree, although it cannot be perfectly balanced.

\paragraph{Wide-ResNet.}
For Wide-ResNet, in order to make it suitable for large model training evaluation, we used the variant with width factor 8.
Wide-ResNet is a collection of Bottleneck layers with three convolutions wrapped with a skip connection, and there are four different sizes of Bottleneck layers laid out sequentially.
Therefore, even with minimum imbalance partitioning, it is difficult to perfectly balance stages.

\subsection{Does Imbalance Decrease with Larger Models?}

Specifically for Transformer~\cite{transformer-neurips17}-based models with homogeneous intermediate Transformer layers, the degree of imbalance will decrease \emph{if the number of pipeline stages is held constant}.
This is because the relative size of the embedding layer and language model head will decrease as intermediate Transformer layers increase in number and size.
However, as the model grow larger, the number of pipeline stages will have to be increased simultaneously.
This increases imbalance, because the relative amount of computation in one layer with respect to the amount in one pipeline stage becomes larger.
As such, there is no simple relationship between model size and the amount of imbalance.

\subsection{Alternative Planning Methods}

3D parallelism is the go-to solution for large model training, and our minimum imbalance pipeline partitioning method is optimal for 3D parallelism because we implemented a brute force search.
More advanced planning methods such as Alpa~\cite{alpa-osdi22} exist, but for repetitive language model architectures like GPT-3, Alpa equally allocates Transformer layers to each stage~\cite{alpa-osdi22}, resulting in the same stage imbalance caused by the language modeling head.
Furthermore, it is hard for tensor parallelism to divide operations infinitely; practically at most degree 8 (within one node) due to collective communication overhead.
Thus, in the vast majority of models, there will always be some degree of imbalance between stages due to the minimum granularity of computation time.

\subsection{Experiment Parameters}

\begin{table*}[!t]
	\centering
  \begin{tabular}{lrrrr}
    \toprule
    \textbf{Model} & \textbf{\# Parameters} & \textbf{Global Batch Size} & \textbf{Microbatch Size} & \textbf{\# Microbatches} \\ 
    \midrule
    gpt3-6.7b & 6.7 B & 1024 & 4 & 128 \\
    \bottomrule
  \end{tabular}
  \caption{Experiment Parameters for 3D parallelism experiments on A40 GPUs. Microbatch size is per-pipeline, and there are two data parallel copies of the same pipeline. Thus, global batch size should be calculated as the product of microbatch size and the number of microbatches times two.}\label{tab:appendix-parameters-3d-a40}
\end{table*}

\begin{table*}[!t]
	\centering
  \begin{tabular}{lrrrr}
    \toprule
    \textbf{Model} & \textbf{\# Parameters} & \textbf{Global Batch Size} & \textbf{Microbatch Size} & \textbf{\# Microbatches} \\ 
    \midrule
    gpt3-2.7b & 2.7 B & 1024 & 4 & 256 \\
    bert-huge-uncased & 1.3 B & 256 & 8 & 32 \\
    t5-3b & 3.0 B & 128 & 4 & 32 \\
    bloom-3b & 3.0 B & 512 & 4 & 128 \\
    wide-resnet101 (width factor 8) & 1.5 B & 1536 & 32 & 48 \\
    \bottomrule
  \end{tabular}
  \caption{Experiment Parameters for eight-stage pipeline parallelism experiments on A40 GPUs. Model variant names are as described in Torch Vision~\cite{pytorch} or Huggingface Hub~\cite{huggingface}.}\label{tab:appendix-parameters-pp-a40}
\end{table*}

\begin{table*}[!t]
	\centering
  \begin{tabular}{lrrrr}
    \toprule
    \textbf{Model} & \textbf{\# Parameters} & \textbf{Global Batch Size} & \textbf{Microbatch Size} & \textbf{\# Microbatches} \\ 
    \midrule
    gpt3-xl & 1.3 B & 512 & 4 & 128 \\
    bert-huge-uncased & 1.3 B & 256 & 8 & 32 \\
    t5-3b & 3.0 B & 128 & 4 & 32 \\
    bloom-3b & 3.0 B & 512 & 4 & 128 \\
    wide-resnet101 (width factor 8) & 1.5 B & 1536 & 64 & 24 \\
    \bottomrule
  \end{tabular}
  \caption{Experiment Parameters for Pipeline Parallelism Experiments on A100 PCIe GPUs. Model variant names are as described in Torch Vision~\cite{pytorch} or Huggingface Hub~\cite{huggingface}.}\label{tab:appendix-parameters-pp-a100}
\end{table*}

Tables~\ref{tab:appendix-parameters-3d-a40},~\ref{tab:appendix-parameters-pp-a40}, and~\ref{tab:appendix-parameters-pp-a100} list model variant names and experiment parameters for our experiments.
Model names and configurations for GPT-3 were taken as is from the original model publication~\cite{gpt3}.
Especially, model names and configurations for BERT and T5 were directly taken from the Huggingface Hub pretrained model zoo, except for the huge variant of BERT, which we created to have hidden dimension 2048.
Wide-ResNet was based on Torch Vision~\cite{pytorch} but scaled up following the model's original publication~\cite{wideresnet-bmvc16} using its width factor parameter.
The unit time parameter $\tau$ was set to 1 ms for all experiments.

\section{Pipeline Energy Minimization is NP-hard}\label{sec:appendix-np-hardness}

The Pipeline Energy Minimization problem is stated again for convenience:
\begin{equation}
  \begin{aligned}
    \label{eq:appendix-pem}
    \min_{F}      \quad & \textrm{Energy}(F) \\
    \textrm{s.t.} \quad & \textrm{Time}(F) \le T'
  \end{aligned}
\end{equation}
where $F$ is the frequency assignment for each pipeline computation $i \in \mathcal{G}$, $T'$ is the straggler pipeline's iteration time.
The decision problem corresponding to Equation~\ref{eq:appendix-pem} asks whether it is possible to find frequency assignment $F$ such that the total energy consumption is minimized while the iteration time of pipeline $\mathcal{G}$ is no longer than the straggler's iteration time.
We denote this problem \textsc{pem}.

In the following, we show that a simplification of \textsc{pem} is NP-hard by reduction from the 0/1 Knapsack problem, which makes \textsc{pem} NP-hard.

\subsection{One Stage Two Frequencies Simplification}\label{sec:appendix-one-stage}

A simplification of \textsc{pem} is considering the case where there is only one pipeline stage and two frequencies to choose from.
 
For each pipeline computation $i \in \{ 1, 2, \ldots, n \}$, we can set the GPU frequency to either the lowest value or the highest, denoted as $[L, H]$ respectively.
Choosing different frequencies will lead to different execution time and energy consumption.
That is, if $i$ is chosen to execute at frequency $L$, it will take $t_{i}(L)$ time and $e_{i}(L)$ energy.
On the other hand, if $i$ executes at frequency $H$, it takes $t_{i}(H)$ time and $e_{i}(H)$ energy.
The time and energy consumption of $i$ are rational numbers, as they are rounded up to $\tau$.

Our goal is to minimize the energy consumption of executing all computations while satisfying the time constraint.
Specifically, given a time deadline $T'$, we want to pick a subset of operations $J\subseteq \{1, 2, \ldots, n\}$ and assign them to execute at the lowest frequency $L$ and execute the rest of the operations with the highest frequency $H$, such that the total time needed to execute all operations is smaller than or equal to the deadline:
\[\sum^n_{i=1} ( X_i t_{i}(L) + (1 - X_i) t_{i}(H) ) \leq T'\]
where $X_i$ is a 0/1 indicator variable where $X_i = 1$ if $i \in J$ and $X_i = 0$ otherwise.
Under this time constraint, the goal is to minimize the total energy consumption of executing all computations:
\[\sum^n_{i=1} ( X_i e_{i}(L) + (1 - X_i) e_{i}(H) ).\]

Formally, we denote this problem as

\[\textsc{pem-1d}(T=(T_L, T_H), E=(E_L, E_H), T', EC)\]
where $T_L=[t_{1}(L),\ldots, t_{n}(L)]$, $T_H=[t_{1}(H),\ldots, t_{n}(H)]$ are execution time vectors for low and high frequency respectively, $E_L=[e_{1}(L),\ldots, e_{n}(L)]$, $E_H=[e_{1}(H),\ldots, e_{n}(H)]$ are energy consumption vectors for low and high frequency respectively, and $EC$ the target energy consumption.

$\textsc{pem-1d}$ returns true if and only if there exists a subset of operations $J\subseteq \{1, 2, \ldots, n\}$ such that $\sum^n_{i=1} ( X_i t_{i}(L) + (1 - X_i) t_{i}(H) ) \leq T'$ and $\sum^n_{i=1} ( X_i e_{i}(L) + (1 - X_i) e_{i}(H) ) \leq EC$.

\subsection{0/1 Knapsack Problem}

Consider two length $n$ arrays containing positive integer weights $W = (w_1, w_2, \ldots, w_n)$ and values $V = (v_1, v_2,  \ldots, v_n)$ where the $i$th item has weight $w_i \in \mathbb{Q}^+$ and value $v_i\in \mathbb{Q}^+$, and a knapsack with weight capacity $C\in \mathbb{Q}^+$.

The goal is to pick a subset of items $S\subseteq \{1, 2, \ldots, n\}$, such that the total weight of the chosen items is less than or equal to the weight capacity: $\sum_{i\in S} w_i\le C$.
Under this constraint, the goal is to maximize the total value of items in the knapsack: $\sum_{i\in S} v_i$.

Formally, we denote the decision problem of 0/1 Knapsack as
\[\textsc{knapsack}(W[1,\ldots,n], V[1,\ldots,n], C, P)\]
where $P \in \mathbb{Q}$ is the target value.

$\textsc{knapsack}$ returns true if and only if there exists a subset of items $S\subseteq \{1, 2, \ldots, n\}$ such that $\sum_{i\in S} w_i\le C$ and $\sum_{i\in S} v_i \ge P$.

It is well known that $\textsc{knapsack}$ is NP-hard. 

\subsection{NP-hardness Proof}

\begin{theorem}
\textsc{pem-1d} is NP-hard.
\end{theorem}
\begin{proof}
We will show that $\textsc{knapsack} \leq_p \textsc{pem-1d}$, i.e., the 0/1 Knapsack problem is polynomial-time reducible to the simplified pipeline energy minimization problem.
Reduction function $f$ takes $(W[1,\ldots,n], V[1,\ldots,n], C, P)$ as input and does the following:
\begin{denseenum}
  \item Construct $n$ computations and empty vectors $T_L, T_H, E_L, E_H$.
  \item For $\forall i$, set $t_{i}(L)=w_i$ and append to $T_L$.
  \item For $\forall i$, set $t_{i}(H)=0$ and append to $T_H$. 
  \item For $\forall i$, set $e_{i}(L)=-v_i$ and append to $E_L$.
  \item For $\forall i$, set $e_{i}(H)=0$ and append to $E_H$.
  \item Set $T'$ = $C$ and $EC$ = $-P$.
  \item Output $(T=(T_L, T_H), E=(E_L, E_H), T', EC)$
\end{denseenum}

\paragraph{Correctness Analysis}
If $(W[1,\ldots,n], V[1,\ldots,n], C, P) \in \textsc{knapsack}$, there exists a subset $S$ such that $\sum_{i\in S} w_i\le C$ and $\sum_{i\in S} v_i \ge P$.
Now for $\textsc{pem-1d}(T=(T_L, T_H), E=(E_L, E_H), T', EC)$, select computations that have the same indices as items in $S$ to execute at low frequency $L$, while executing others at high frequency $H$.
Then, for the time constraint, $\sum^n_{i=1} ( X_i t_{i}(L) + (1 - X_i) t_{i}(H) )  = \sum^n_{i=1} X_i t_{i}(L)  = \sum_{i\in S} w_i \leq C = T'$, and for target energy, $\sum^n_{i=1} ( X_i e_{i}(L) + (1 - X_i) e_{i}(H) )= \sum^n_{i=1} X_i e_{i}(L) = \sum_{i\in S} -v_i \leq -P = EC$.

If $(W[1,\ldots,n], V[1,\ldots,n], C, P) \notin \textsc{knapsack}$, there does not exist a subset $S$ such that $\sum_{i\in S} w_i\le C$ and $\sum_{i\in S} v_i \ge P$. There are two possibilities:
either a subset $S$ that satisfies the weight constraint does not exist at all ($w_i > C, \forall i$)
or none of the subsets $S$ that satisfy the weight constraint satisfy $\sum_{i\in S} v_i \ge P$.
For the first possibility, this means all the computations must select the high frequency as the low frequency does not satisfy the time constraint.
Then total energy consumption is 0, which is larger than $EC = -P$ since $P \in \mathbb{Q}^+$.
For the second possibility, for all subsets $S$, $\sum_{i\in S} v_i < P$, which means that for all subsets of computations $\sum^n_{i=1} ( X_i e_{i}(L) + (1 - X_i) e_{i}(H) ) = \sum_{i\in S} -v_i > -P = EC$, so none of them satisfy the energy constraint.

\paragraph{Efficiency Analysis}
Step 1--5 each takes $O(n)$ time.
Step 6 takes $O(1)$ time.
Finally, step 7 takes $O(n)$ time. 

Therefore, the function $f$ takes $O(n)$ time, which is polynomial time w.r.t the input size.

\end{proof}

\section{Continuous Relaxation}\label{sec:appendix-relaxation}

\begin{figure}[t]
  \centering
  \includegraphics[width=0.42\textwidth]{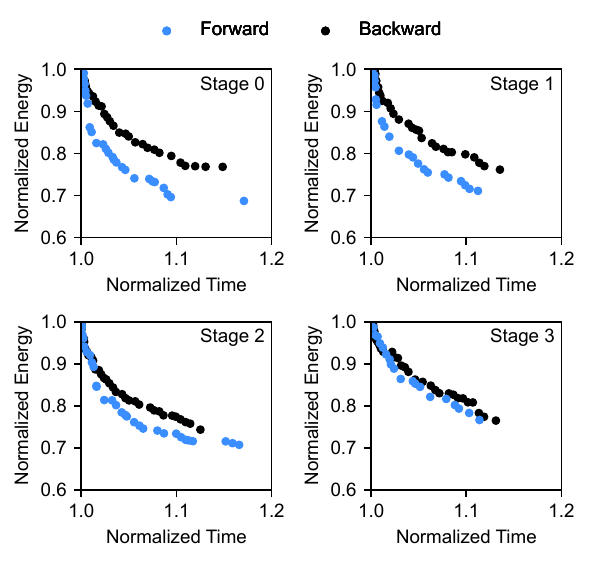}
  \caption{Pareto-optimal (time, energy) choices for forward and backward computations in each stage of GPT-3 0.3B on NVIDIA A40 GPUs.}\label{fig:appendix-exponential-relaxation}
\end{figure}

The Pipeline Energy Minimization problem is NP-Hard, as we have proved in Appendix~\ref{sec:appendix-np-hardness}.
Thus, we perform continuous relaxation for the problem by fitting an exponential function to each Pareto-optimal (time, energy) points for each forward and backward computation in each stage.
Furthermore, the exponential function captures the nature of \emph{diminishing returns} well, in that the amount of energy consumption needed to reduce computation time by a constant amount increases multiplicatively.

Figure~\ref{fig:appendix-exponential-relaxation} visualizes Pareto-optimal (time, energy) measurements for GPT-3 0.3B on A40 GPUs, showing that the exponential function is a natural fit for data.
This was our consistent observation across different GPUs and models in our evaluation.

\section{Full Details of GetNextSchedule}\label{sec:appendix-min-cut}

For the sake of presentation, we made a simplifying statement in the body of the paper that computations only speed up by $\tau$.
However, since we aim to speed up all critical paths by precisely $\tau$, speeding up more than one computation from a critical path allows other computations on that critical path to be slowed down.
We always take such slowdown opportunity because it will decrease energy consumption.

In the following, we describe the procedure of annotating edges with flow capacities and then solving the problem with maximum flow.

\subsection{Generating Capacity DAG}

On the computation DAG, we first remove all the computations that are not on any of the critical paths and construct the \emph{Critical DAG}. 
We would like to find a set of edges $I^+$ to speed up by $\tau$ and $I^-$ to slow down by $\tau$ on the Critical DAG, such that the total energy consumption increases minimally.
This can be described as the following problem:
\begin{equation}
  \label{eq:design-reduction-cost}
  \min_{I^+,~I^-}      \quad \sum_{i \in I^+}{e_i^+} - \sum_{i \in I^-}{e_i^-},
\end{equation}
where $e_i(t_i)$ is the exponential function fit to Pareto-optimal (time, energy) measurements for computation $i$, $e_i^+ = e_i(t_i - \tau) - e_i(t_i)$ is the extra amount of energy needed to speed up $i$ by $\tau$ and $e_i^- = e_i(t_i) - e_i(t_i + \tau)$ is the energy saved from slowing down $i$ by $\tau$.

An important fact that leads us to the solution is that (i) the problem of finding the set of edges to modify such that energy increase is minimized (\ie, solving Equation~\ref{eq:design-reduction-cost}) coincides with (ii) finding the \emph{minimum cut} of a DAG whose lower and upper bound flow capacities are defined as
\begin{equation}
  \begin{aligned}
    \label{eq:design-edge-capacity}
    (l_i, u_i) =
    \begin{cases}
      (0, e_i^+)      & \textrm{if } t_i \textrm{ is longest possible (slowest)} \\
      (e_i^-, \infty) & \textrm{if } t_i \textrm{ is shortest possible (fastest)} \\
      (e_i^-, e_i^+)  & \textrm{otherwise.}
    \end{cases}
  \end{aligned}
\end{equation}
This equivalence was given by the theoretical works of Phillips and Dessouky~\cite{pd-mnsc77,pd-caie16}.

Thus, we construct the \emph{Capacity DAG} from the Critical DAG by annotating its edges with flow capacities given by Equation~\ref{eq:design-edge-capacity}.
The capacity of an $S-T$ cut ($S$ is the set of nodes on the source side and $T$ the sink side) on the Capacity DAG with $S \rightarrow T$ edges $I^+$ and $T \rightarrow S$ edges $I^-$ is identical to the objective in Equation~\ref{eq:design-reduction-cost}.
Then, we use the Edmonds-Karp maximum flow algorithm~\cite{edmondskarp1972} to find the minimum cut of the Capacity DAG.
Finally, after the minimum cut has been identified from the Capacity DAG, edges in $I^+$ are sped up by $\tau$ and those in $I^-$ are slowed down by $\tau$, ultimately reducing the length of every critical path exactly by $\tau$ with the smallest possible energy increase.

\subsection{Max Flow Algorithm on the Capacity DAG}\label{sec:appendix-max-flow}

A characteristic of our Capacity DAG that precludes the direct application of well-known max flow algorithms is that some edges also have flow \emph{lower bounds}, asserting that at least a certain amount of flow must pass through in the edge.
However, the Max Flow Min Cut theorem by Ford and Fulkerson holds for the case of non-zero flow lower bounds (See Chapter 1, Section 9 of~\cite{fordfulkerson}), allowing us to find the minimum cut (which is equivalent to the minimum energy modification set) using any maximum flow algorithm.
We adopt an approach that adds dummy source/sink nodes to create a DAG that has zero flow lower bounds, finds the maximum flow on the new DAG, and extracts the flow so that it corresponds to a flow in the original DAG~\cite{boundedflow}.
The algorithm is given in Algorithm~\ref{algo:bounded-max-flow}.

\begin{algorithm}[t]
  \algrule{}

  \SetKwFunction{EdmondsKarp}{EdmondsKarp}
  \SetKwFunction{FlowValue}{FlowValue}

  \KwIn{Directed Acyclic Graph $G = (V, E)$\newline
        Source node $s \in V$ and sink node $t \in V$\newline
        Lower and upper bounds $l(e)$, $u(e)$ for $\forall e \in E$
  }
  \KwOut{A maximum feasible flow of $G$ if it exists}

  \algrule{}

  Construct a new graph $G' = (V', E')$ by adding new source and sink nodes $s'$ and $t'$, edges from $s'$ to each node in $V$, edges from each node in $V$ to $t'$, and an edge from $t$ to $s$\;

  \Comment{Define the capacity $c'(e)$ of each edge $e \in E'$}:
  \For{$v \in V$}{
      $c'(s'v) \leftarrow \sum_{u \in V} l(uv)$\;
      $c'(vt') \leftarrow \sum_{w \in V} l(vw)$\;
  }
  \For{$uv \in E$}{
      $c'(uv) \leftarrow u(uv) - l(uv)$\;
  }
  $c'(ts) \leftarrow \infty$\;

  \Comment{Find the max flow on G'}
  $f' \leftarrow$ \EdmondsKarp($G'$, $c'$)\;

  \Comment{$G$ has a feasible flow if and only if $G'$ has a saturating flow}
  \If{\FlowValue($f'$) $\neq \sum_{v \in V}{c'(s'v)}$}{
    \Return{nil}\;
  }

  \Comment{Convert $f'$ to a feasible flow $f$ in $G$}
  \For{$uv \in E$}{
    $f(uv) \leftarrow f'(uv) + l(uv)$\;
  }

  \Comment{Construct residual graph and improve $f$ to max flow}
  \For{$uv \in E$}{
    $c(uv) \leftarrow u(uv) - f(uv)$\;
    $c(vu) \leftarrow f(vu) - l(vu)$\;
  }
  \Return{\EdmondsKarp($G$, $c$)}\;

  \algrule{}

  \caption{Maximum Flow with Lower Bounds}\label{algo:bounded-max-flow}
\end{algorithm}

\section{Proof for Polynomial Runtime}\label{sec:appendix-polytime}

Perseus enumerates the entire \revised{time--energy} frontier one by one, and the runtime of one iteration is polynomial time with respect to the number of stages $N$ and the number of microbatches $M$.
Thus, determining whether the entire algorithm runs in polynomial time reduces to whether the worst case number of iterations is polynomial with respect to $N$ and $M$.
While for general DAGs the maximum number of points on the frontier can be exponential with respect to the size of the DAG~\cite{skutella-thesis}, here we prove that under mild assumptions for DAGs that represent pipeline schedules, the number of iterations is $O(N + M)$.
The assumptions are valid for all pipeline schedules known to the authors, including GPipe~\cite{gpipe-neurips19} and 1F1B~\cite{pipedreamflush-icml21}.

\begin{theorem}
  For DAGs that represent pipeline schedules, the number of iterations needed is $O(N + M)$.
\end{theorem}
\begin{proof}
Since we always reduce iteration time by $\tau$, the number of iterations is
\[\frac{t_{\max} - t_{\min}}{\tau}\]
where $t_{\max}$ and $t_{\min}$ are the maximum and minimum possible iteration time, respectively.

Assume that any pipeline schedule representing one iteration of training has a prologue, a steady state, and an epilogue.
The prologue is when the pipeline starts from an empty state and is gradually filled with pipeline computations, while the epilogue is when the pipeline is drained to reach an empty state.
It is easy to see that the number of pipeline computations on the critical path of both the prologue and epilogue is $O(N)$, as deeper pipelines (larger $N$) take longer to fill.
On the other hand, the steady state of the pipeline is when the pipeline is completely filled, and the number of pipeline computations in any simple path through the steady state of the DAG is $O(M)$.
Therefore, the total number of pipeline computations in the critical path of the entire DAG is $O(N + M)$.

$t_{\max}$ and $t_{\min}$ can be constructed by multiplying the number of computations with the average execution time of a computation.
Computations are executed with frequencies $f_{\min}$ and $f_{\max}$, respectively, and thus the multipliers of $N$ and $M$ do not cancel out when $t_{\max} - t_{\min}$ is evaluated.
Therefore, $t_{\max} - t_{\min}$, and hence $(t_{\max} - t_{\min})/\tau$, is $O(N + M)$.
\end{proof}

\begin{listing}
  \begin{minted}[]{python}
# In the framework's pipeline execution engine:
from perseus.client import profiler, controller

def train_step(model, dataloader):
    ...

    for instuction in pipeline_schedule:
        if isinstance(instruction, Forward):
            controller.set_speed("forward")
            profiler.begin("forward")
            # Run forward on microbatch
            profiler.end("forward")
        elif isinstance(instruction, Backward):
            controller.set_speed("backward")
            profiler.begin("backward")
            # Run backward on microbatch
            profiler.end("backward")
    ...
\end{minted}
\caption{Perseus client API integration example.}\label{lst:appendix-integration}
\end{listing}

\section{Perseus Client Integration Example}\label{sec:appendix-integration}

Listing~\ref{lst:appendix-integration} shows an example integration of Perseus's Client API (first three rows in Table~\ref{tab:implementation-client-api}) with a hypothetical (but typical) training framework's pipeline execution engine.

A typical structure of a pipeline execution engine is to have \emph{instructions} for each distinct operations in the pipeline, including not only forward and backward executions, but also P2P and collective communications, and implement a handler for each instruction.
Therefore, framework developers can wrap such handlers with the Perseus client APIs to mark their beginning and end.

\section{Time--Energy Frontiers}\label{sec:appendix-all-frontiers}

\begin{figure}[!t]
	\centering
	\hfil
	\subfloat[][BERT]{
		\includegraphics[width=0.45\linewidth]{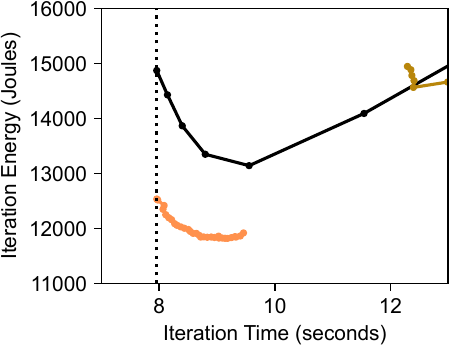}
	}
	\hfil
	\subfloat[][T5]{
		\includegraphics[width=0.45\linewidth]{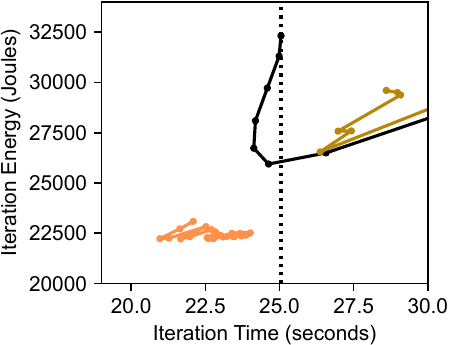} 
	}
	\hfil
	\subfloat[][Bloom]{
		\includegraphics[width=0.45\linewidth]{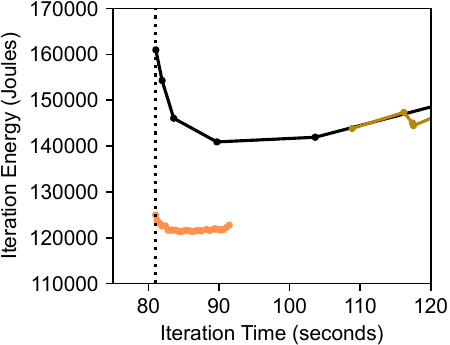}
	}
	\hfil
	\subfloat[][Wide-ResNet]{
		\includegraphics[width=0.45\linewidth]{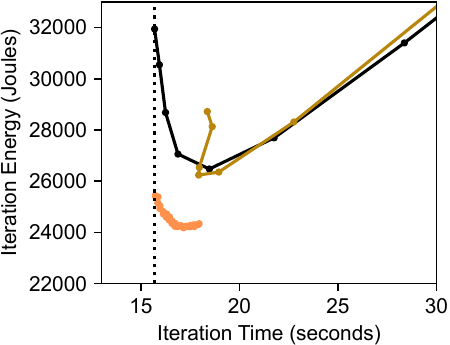}
	}
	\hfil
  \caption{Eight stage pipeline parallelism on A40.}\label{fig:appendix-all-frontiers-pp8-a40}
\end{figure}

Figure~\ref{fig:appendix-all-frontiers-pp8-a40} shows the \revised{time--energy} frontiers achieved by Perseus and the two baseline approaches for the rest of the workloads ran with eight stage pipeline parallelism, measured in NVIDIA A40 GPUs.

T5 shows an interesting frontier due to the hardware topology of our A40 machine setup: Each node has four GPUs and NVLink connects GPUs 0 and 1, and 2 and 3; GPUs 1 and 2 must communicate through the NUMA interconnect; Finally, different nodes are connected with Infiniband only adjacent to GPUs 0 and 1 (data to and from GPUs 2 and 3 must also go through the NUMA interconnect).
The implication of this heterogeneous GPU interconnect is that if more than one P2P communications that need to go through the NUMA interconnect happen at the same time, contention happens and both of data transfers slow down significantly.
However, Perseus's plan reduces this contention, overall decreasing iteration time noticeably.
Yet, contention and noisy communication latencies still exist, leading to a noisy frontier.

For many ZeusPerStage lines, energy fluctuates significantly when iteration time increases due to ZeusPerStage being unaware of critical paths.
Balancing the forward computation time between stages could even let the modified stages take over the critical path.
As a result, the iteration time increases, which increases energy bloat, and more energy is spent on blocking on communication (\S\ref{sec:design-problem-formulation}).
When the decreased energy on computation fails to cover the increased energy on P2P communication, total energy increases.

\begin{figure}[!t]
	\centering
	\hfil
	\subfloat[][BERT]{
		\includegraphics[width=0.45\linewidth]{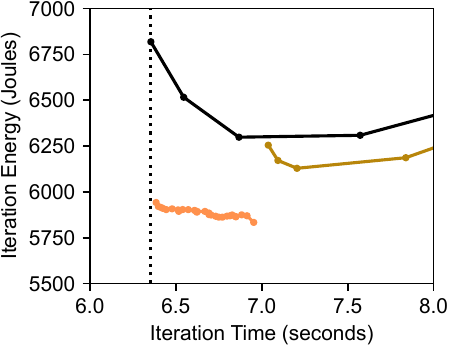}
	}
	\hfil
	\subfloat[][T5]{
		\includegraphics[width=0.45\linewidth]{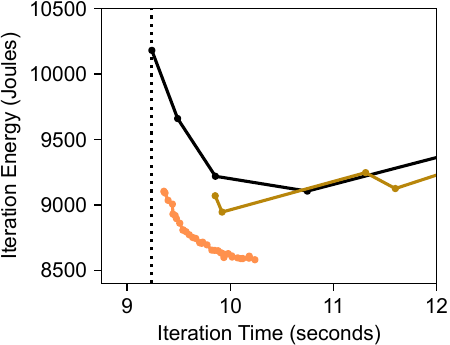} 
	}
	\hfil
	\subfloat[][Bloom]{
		\includegraphics[width=0.45\linewidth]{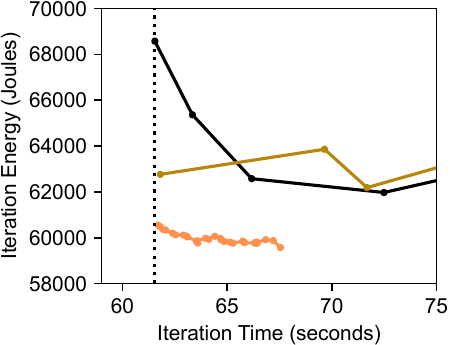}
	}
	\hfil
	\subfloat[][Wide-ResNet]{
    \includegraphics[width=0.45\linewidth]{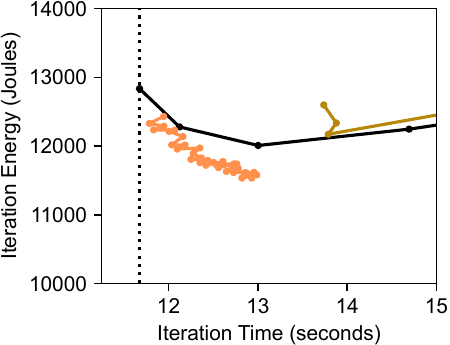}
	}
	\hfil
  \caption{4 stage pipeline parallelism on A100 PCIe.}\label{fig:appendix-all-frontiers-pp4-a100}
\end{figure}

Figure~\ref{fig:appendix-all-frontiers-pp4-a100} shows the \revised{time--energy} frontiers achieved by Perseus and the two baseline approaches for the rest of the workloads, measured with four stage pipeline parallelism in NVIDIA A100 PCIe GPUs.
Wide-ResNet has a noisy frontier because the variability in microbatch loading time introduces noise in the end-to-end iteration time when computations are tightly packed by Perseus.
This was not pronounced in A40 GPUs because compared to A100 PCIe, computation is slower, but data loading time is similar.
Thus, the noise in data loading time becomes more noticeable in A100 PCIe.

\cleardoublepage

\end{document}